\documentclass{article}
\usepackage{tweaklist}

\usepackage[final]{neurips_2021}

\usepackage[utf8]{inputenc} % allow utf-8 input
\usepackage[T1]{fontenc}    % use 8-bit T1 fonts
\usepackage{hyperref}       % hyperlinks
\usepackage{url}            % simple URL typesetting
\usepackage{booktabs}       % professional-quality tables
\usepackage{amsfonts}       % blackboard math symbols
\usepackage{nicefrac}       % compact symbols for 1/2, etc.
\usepackage{microtype}      % microtypography
\usepackage{xcolor}         % colors

\usepackage{amsmath}
\usepackage{amsthm}
\usepackage{amssymb}
\usepackage{algorithm}
\usepackage{algorithmic}
\usepackage{subfig}
\usepackage{tikz}
\usetikzlibrary{matrix,chains,positioning,decorations.pathreplacing,arrows}
\usetikzlibrary{bayesnet}

\usetikzlibrary{backgrounds}
\usetikzlibrary{mindmap,trees}

\newcommand{\R}{\mathbb{R}}
\newcommand{\E}{\mathbb{E}}
\newcommand{\prob}{\mathbb{P}}
\newcommand{\domain}{\mathbb{X}}
\newcommand{\X}{\mathbb{X}}
\newcommand{\eifn}{\textnormal{EI-FN}}

\newcommand{\I}[2]{I_{#1,#2}}
\newcommand{\f}{f}
\newcommand{\h}{h}
\newcommand{\obj}{g}

\DeclareMathOperator*{\argmax}{argmax}

\newtheorem{theorem}{Theorem}
\newtheorem{proposition}{Proposition}

\makeatletter
\def\renewtheorem#1{%
  \expandafter\let\csname#1\endcsname\relax
  \expandafter\let\csname c@#1\endcsname\relax
  \gdef\renewtheorem@envname{#1}
  \renewtheorem@secpar
}
\def\renewtheorem@secpar{\@ifnextchar[{\renewtheorem@numberedlike}{\renewtheorem@nonumberedlike}}
\def\renewtheorem@numberedlike[#1]#2{\newtheorem{\renewtheorem@envname}[#1]{#2}}
\def\renewtheorem@nonumberedlike#1{  
\def\renewtheorem@caption{#1}
\edef\renewtheorem@nowithin{\noexpand\newtheorem{\renewtheorem@envname}{\renewtheorem@caption}}
\renewtheorem@thirdpar
}
\def\renewtheorem@thirdpar{\@ifnextchar[{\renewtheorem@within}{\renewtheorem@nowithin}}
\def\renewtheorem@within[#1]{\renewtheorem@nowithin[#1]}
\makeatother

\title{Bayesian Optimization of Function Networks}

\author{%
  Raul Astudillo\\
  Cornell University\\
  \texttt{ra598@cornell.edu}\\
  \And
  Peter I. Frazier\\
  Cornell University\\
  \texttt{pf98@cornell.edu}\\
}

\begin{document}

\maketitle

\begin{abstract}
We consider Bayesian optimization of the output of a network of functions, where each function takes as input the output of its parent nodes, and where the network takes significant time to evaluate. Such problems arise, for example, in reinforcement learning, engineering design, and manufacturing.  While the standard Bayesian optimization approach observes only the final output, our approach delivers greater query efficiency by leveraging information that the former ignores: intermediate output within the network. This is achieved by modeling the nodes of the network using Gaussian processes and choosing the points to evaluate using, as our acquisition function, the expected improvement computed with respect to the implied posterior on the objective. Although the non-Gaussian nature of this posterior prevents computing our acquisition function in closed form, we show that it can be efficiently maximized  via sample average approximation. In addition, we prove that our method is asymptotically consistent, meaning that it finds a globally optimal solution as the number of evaluations grows to infinity, thus generalizing previously known convergence results for the expected improvement. Notably, this holds even though our method might not evaluate the domain densely, instead leveraging problem structure to leave regions unexplored.
Finally, we show that our approach dramatically outperforms standard Bayesian optimization methods in several synthetic and real-world problems.
\end{abstract}

\section{Introduction}
\label{sec:intro}

We consider Bayesian optimization (BO) of objective functions defined by  a series of time-consuming-to-evaluate functions, $\f_1,\ldots, \f_K$,  arranged in a directed acyclic network, so that each function takes as input the output of its parent nodes. As we detail below, these problems arise in  BO-based policy search in reinforcement learning \citep{lizotte2007automatic},
optimization of complex systems modeled via simulation, and calibration of time-consuming physics-based models.
%design of pandemic mitigation strategies,
%and optimization of queuing networks.
%\citep{buchholz2005enhancing, pichitlamken2006sequential}.
%including applications to manufacturing.

%computer-aided materials design 
%\citep{packwood2017bayesian};
%aircraft design \citep{chaudhuri2018multifidelity}; 

To illustrate, we introduce a running example of vaccine manufacturing \citep{sekhon2011biosimilars}, focusing on the portion of the manufacturing process that uses live cells to produce proteins needed in a vaccine.
It begins with a cell culture, in which living cells are grown and used as ``factories'' to produce proteins. This process is controlled by a vector, $x_1$, containing the temperature, pH, and CO$_2$ content used when growing these cells. The output of this process is the quantity of the desired protein $y_1 = \f_1(x_1)$, i.e., the {\it yield} of this step, along with other byproducts.  This output is passed into a second process, purification, which removes byproducts and is controlled by a vector $x_2$ comprising temperature, pressure, and flow rate. The yield of the desired protein from this second step is $y_2 = \f_2(x_2, y_1)$. This output enters a third step, formulation, in which we formulate the raw protein into a form that can be distributed as controlled by a third set of parameters. This determines the yield of the overall process $y_3 = \f_3(x_3,y_2)$. We wish to choose $(x_1, x_2, x_3)$ to maximize overall protein yield. 
This problem is summarized as a function network in Figure~\ref{fig:vaccine}.
%While we focus here on vaccine manufacturing, similar statements will apply to other manufacturing examples.
% Also, while we focus on $K=3$ in biologics manufacturing, a more detailed example could be constructed with $K>3$ 

The problem described above and other similar problems can be tackled with Bayesian optimization (BO), which has been shown to perform well compared to other derivative-free global optimization methods for time-consuming-to-evaluate objective functions \citep{snoek2012practical,frazier2018tutorial}. A standard BO algorithm would fit a Gaussian process (GP) \citep{williams2006gaussian} model on the objective function ($y_3$, which depends on $(x_1, x_2, x_3)$) and use it, along with an acquisition function, to sequentially choose the points to evaluate. Under this standard approach, however, evaluations of the intermediate nodes, $\f_1,\ldots, \f_{K-1}$, would be ignored despite being available when computing the objective function. In the example above, this corresponds to looking only at the yield of the overall process, and not of each individual step. 

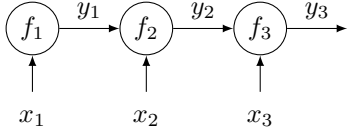
\begin{figure}[h]
  \centering
\begin{tikzpicture}[
init/.style={
  draw,
  circle,
  inner sep=0.7pt,
  minimum size=0.7cm
},
init2/.style={
  circle,
  inner sep=0.7pt,
  minimum size=0.7cm
},
]
\begin{scope}[start chain=1,node distance=8mm]
\node[on chain=1, init] 
  (f1) {$\f_1$};
\node[on chain=1,init]
 (f2) {$\f_2$};
\node[on chain=1,init] (f3) 
  {$ \f_3$};
\node[on chain=1,init2] (f4){};
\end{scope}

\begin{scope}[start chain=2,node distance=8mm]
\node[on chain=2, init2] at (0,-12mm)
(x1) {$x_1$};
\node[on chain=2, init2] 
  (x2) {$x_2$};
\node[on chain=2, init2] 
  (x3) {$x_3$};
\end{scope}

\draw[-latex] (f1) -- (f2)node[pos=0.5,sloped,above] {$y_1$};
\draw[-latex] (f2) -- (f3)node[pos=0.5,sloped,above] {$y_2$};
\draw[-latex] (f3) -- (f4)node[pos=0.5,sloped,above] {$y_3$};

\draw[-latex] (x1) -- (f1);
\draw[-latex] (x2) -- (f2);
\draw[-latex] (x3) -- (f3);
\end{tikzpicture}
\caption{Vaccine manufacturing as a function network.  Protein $y_1 = \f_1(x_1)$ is created, 
then purified with yield $y_2 = \f_2(x_2,y_1)$,
% depending on the yield of the first step and purification-specific controls; 
 and formulated with yield $y_3 = \f_3(x_3,y_2)$.
%depending on the yield of the second step and formulation-specific controls. 
The goal is to find $(x_1,x_2,x_3)$ that maximizes $y_3$.\label{fig:vaccine}}
\end{figure}

In this paper, we introduce a novel BO approach that leverages function network structure for substantially more efficient optimization. This approach models the individual nodes of the network using distinct GPs. This allows incorporating observations of each node's output recursively into a non-Gaussian posterior on the network's overall output. Our approach then chooses the points to evaluate using the expected improvement \citep{jones1998efficient} computed with respect to this implied posterior on the objective function. The non-Gaussian nature of this posterior prevents the expected improvement from having a closed form. However, we show that it can still be efficiently maximized via sample average approximation \citep{kleywegt2002sample}.

Our approach can outperform standard BO by leveraging information from internal nodes unavailable to standard methods. We briefly explain one way this can happen, in the context of the example above.
In vaccine manufacturing, each function $\f_k(x_k, y_{k-1}), \ k= 2, 3$, is bounded between $0$ and $y_{k-1}$ because new protein cannot be created in the purification and formulation stages. Moreover, application experts have a prior on what values for $\f_k(x_k, y_{k-1}) / y_{k-1}$ should be achievable if $x_k$ is set well.
Thus, if we see that $y_3$ is unexpectedly poor, information from intermediate nodes can be extremely valuable:
if $y_1$ and $y_2$ are as expected, then this suggests the problem is with $x_3$;
if $y_1$ is as expected but $y_2$ is poor, then the problem is likely with $x_2$;
and if $y_1$ is poor then the problem is likely with $x_1$. 
(If there is a problem with $x_k$, there may also be a problem with $x_{k'}$, $k'>k$, but we can focus on fixing $x_k$ first.)
Thus, by observing intermediate nodes, we can instantly reduce the effective dimensionality of the input space by a factor of $K=3$.

% We take inspiration from biologics manufacturing in the biopharmaceutical space (cite),
%but the following example is applicable more broadly to other manufacturing processes,
%and even to retaining underrepresented groups in STEM (``leaky pipeline'').  

%Each stage of the manufacturing process, indexed by $i$, has a vector-valued set of inputs $x_i$,
%which are exogenous in this stage,
%and then also has a rate at which partially finished product from the previous stage 
%is passed as input.
%The goal of this stage is to refine the product and pass it on to the next stage.
%Some fraction of the product is lost due to errors or low-quality finishing.
%This fraction is a quantity $p_i(x_i)$ and its evaluation would require running a real-world experiment.

%The rate at which raw materials are provided to the first step in the process is $y_0$.
%Then, the rate at which materials are ouputted from this first step is $y_1 = y_0 p_1(x_1)$.
%Continuing this, the output at stage $n$ is $y_n = y_{n-1} p_n(x_n) = y_0 \prod_{m \le n} p_m(x_m)$.

%For example, in biologics manufacturing, one stage is a purification process (which in reality is actually a sequence of steps).
%This takes as input antibodies produced in the cell culture process mixed with a great deal of other biological material, and refines these out to get pure antibodies.  Some fraction of the antibodies in the original sample are lost in purification.

We show that our method is asymptotically consistent, i.e., that it discovers the global optimum given sufficiently many samples. Remarkably, in contrast with most BO methods, it may do so without measuring densely over the feasible domain, instead leveraging function network structure to exclude regions as unnecessary to explore. This indicates the power of function network structure to improve query efficiency.

We demonstrate through numerical experiments that access to additional information available in a problem formulated as a function network can dramatically accelerate optimization. We study four synthetic problems and four real-world problems: a manufacturing problem similar in spirit to the vaccine example above, an active learning problem  with a robotic arm, and two problems arising in epidemiology, one calibrating an epidemic model and the other designing a testing strategy to control the spread of COVID-19. Our method significantly outperforms competing methods that utilize less information, in some cases by $\sim$5\% and in other cases by several orders of magnitude.

% The remainder of this paper is organized as follows. 
% \S\ref{sec:related} reviews related work. \S\ref{sec:problem} formalizes our problem setting and describes standard BO. 
% \S\ref{sec:our_approach} describes our approach. 
% \S\ref{sec:experiments} presents numerical experiments, and
% \S\ref{sec:conclusion} concludes.

\section{Related Work}
\label{sec:related}
Our work occurs within BO, a framework for global optimization of expensive-to-evaluate black-box functions that originated with the work of \citet{zhilinskas1975single} and \citet{movckus1975bayesian},  and has recently become popular due to its remarkable performance in hyperparameter tuning of machine learning algorithms \citep{snoek2012practical, swersky2013multi,wu2019practical}. We refer the reader to \citet{shahriari16} and \citet{frazier2018tutorial} for modern introductions to BO.

Our approach can be catalogued as a \textit{grey-box} BO method since it does not treat the objective function entirely as a black box, and instead exploits known structure to improve sampling efficiency. Other examples of grey-box BO approaches include multi-fidelity BO \citep{kandasamy17a,wu2019practical}, which leverages cheaper approximations of the objective function; BO of objective functions that are the integral of an expensive-to-evaluate integrand \citep{williams2000sequential,toscano2018bayesian,cakmak2020risk}; BO of objective functions that are a sum of squared errors \citep{uhrenholt2019efficient}; and, more generally, BO of objective functions that are the composition of an expensive-to-evaluate inner function and a known inexpensive-to-evaluate outer function \citep{astudillo2019bayesian}, of which our work can be seen as a significant generalization. We refer the reader to \citet{astudillo2021thinking} for a tutorial on grey-box BO.

Our work is also closely related to \citet{marque2019efficient}, which proposes a method for efficient sequential experimental design of nested computer codes, also using GPs. This work focuses on the case where there are only two node functions, and one takes as input the output of the other. In contrast with our work, the goal of the proposed method is to learn the output code as accurately as possible within a limited evaluation budget, but optimization is not pursued.

 Optimization of composite (a.k.a nested) functions has also been considered in the gradient-based optimization literature \citep{shapiro2003class, drusvyatskiy2019efficiency,charisopoulos2019low, balasubramanian2020stochastic}. In contrast with ours, these works assume that evaluations are inexpensive, and often also some form of convexity, along with availability of gradients. However, this literature is similar in spirit to ours in that information from inner functions improves efficiency.

Function networks arise in many application areas. 
While these applications have not, to our knowledge, been previously formulated as specific instances of the general function network model we propose, their literatures are nonetheless relevant to our work. 
% The closest general model in the literature is \citet{cramer1994problem} from the multidisciplinary optimization literature, described below.
\begin{itemize}
\item In engineering and aerospace design, function networks arise in multidisciplinary optimization \citep{cramer1994problem,amaral2014decomposition,benaouali2019multidisciplinary}, where simulators focusing on different physical laws are coupled into a function network. For example, a simulation of airflow over a wing may output the forces on the wing to another simulation of mechanical stress on the wing structure.
% While there are general-purpose multidisciplinary design optimization methods \citep{cramer1994problem}, most rely on gradients (despite the fact that many aerospace codes do not provide gradients) and we are unaware of a Bayesian optimization method from this literature designed to exploit function network structure.
\item In drug discovery and materials design, function networks arise from the 
sequence-structure-function  \citep{sadowski2009sequence}
and composition-structure-property 
\citep{hattrick2016perspective}
paradigms.
Here, decision variables % $x_1$ 
(composition, e.g., what fraction of what raw materials are used) determine the structure of the material 
% $y_1 = \f_1(x_1)$ 
(how the atoms combine together), which in turn determines how the material behaves (properties).%, $\f_2(y_1)$. 
% Measuring the structure in addition to the properties gives materials scientists knowledge that helps them optimize.
% \item Function networks arise in the estimation of metabolic networks in biology, in which chemical reactions modulated by enzymes serve as nodes that supply molecular species to other chemical reactions, as well as the estimation of kinetic rates for such parameter in nature. 
\item Function networks arise in the design of queuing networks \citep{fu2017history,arcelli2020exploiting}. This includes manufacturing systems \citep{ghasemi2018review}, where the partially-processed output of one workstation is input to the next workstation, the design of service systems \citep{wang2020metamodel} such as hospitals and airport security checkpoints, and choosing traffic signal timings for a city's street network \citep{osorio2013simulation}.
\item Finally, function networks arise in reinforcement learning \citep{sutton2018reinforcement} and Markov decision processes 
\citep{puterman1990markov},
where the transition kernel transforms the state variable at the start of a time range to another state variable at the end of this range. This outputted state variable is the input to the next time range. 
\S\ref{sec:fetch} shows an example.
\end{itemize}

\section{Problem Setting}
\label{sec:problem}
We consider objective functions evaluated via a series of functions, $\f_1,\ldots, \f_K$, arranged in a directed network so that each function in this network takes as input the output of its parent nodes, and assume that evaluating this collection of functions is time-consuming.  The network structure is encoded as a directed graph with nodes $V = \{1, \ldots,K$\} and directed edges $E = \{(j, k) : \f_k \textnormal{ take as input the output of } \f_j\}$. We assume that this graph is acyclic and has a single leaf node  whose output is the objective to optimize.

Let $J(k)$ denote the set of parent nodes of node $k$\footnote{We allow $J(k) = \emptyset$, which corresponds to the root node(s).}. 
We assume, without loss of generality, that the nodes are ordered such that $j < k$ for all $j\in J(k)$. In addition to the output of its parent nodes, we assume that each function, $\f_k$, takes as input a (potentially empty) subset, $I(k) \subset \{1,\ldots, D\}$, of the components of the decision vector $x \in \mathbb{R}^D$. 

Let $\h_1(x), \ldots, \h_K(x)$ denote the values of the $K$ nodes in the function network when it is evaluated at $x$. These values are defined recursively by
\begin{equation*}
    \h_k(x) = \f_k\left(x_{I(k)}, \h_{J(k)}(x)\right), \ k=1,\ldots, K,
\end{equation*}
i.e., by evaluating each function in the network as the values of its parent nodes become available.
This recursion is well-defined by our assumption that the graph is acyclic. The objective function is then $g=h_K$, and the goal is to solve 
\begin{equation*}
  \max_{x \in \mathbb{X}} g(x),  
\end{equation*}
where $\mathbb{X}\subset\R^D$ is a simple compact set, such as a hyper-rectangle.

The standard BO approach to this problem models $g$ using a GP prior distribution. This approach iteratively chooses the next point at which to evaluate $\obj$ as follows. Given $n$ observations of the objective function,  $g(x_1), \ldots, g(x_n)$, it computes the posterior distribution on $\obj$, which is itself a GP. It then uses this posterior distribution to compute an acquisition function \citep{frazier2018tutorial} that quantifies the value of the information that would result from a function evaluation at any given point. Finally, it chooses the next point to evaluate, $x_{n+1}$, as the point that maximizes this acquisition function. Notably, although the outputs of $f_1, \ldots, \f_{K-1}$ would be observed as part of this approach, these evaluations would be ignored by standard BO 
when calculating the posterior distribution and corresponding acquisition function.

\section{Bayesian Optimization with Full Network Observations}
\label{sec:our_approach}
This section describes our approach. Like standard BO, it consists of a statistical model and an acquisition function. Unlike standard BO, however, our approach leverages the network structure of the problem by utilizing intermediate outputs within the network. As we describe below, this is achieved by modeling the node functions, $\f_1,\ldots, \f_K$,  as GPs, which in turn implies a non-Gaussian distribution on $\obj$ (\S\ref{sec:stat_model}). Our acquisition function is the expected improvement under this posterior distribution (\S\ref{sec:eifn}), which no longer has a closed form and thus we maximize it via sample average approximation (\S\ref{sec:eifn_saa}). We end up this section by proving that our acquisition function is asymptotically consistent despite not necessarily sampling $\domain$ densely (\S\ref{sec:ac}).
\subsection{Statistical Model}
\label{sec:stat_model}
Instead of modeling $\obj$ directly,  we model $\f_1,\ldots, f_K$, as drawn from independent GP prior distributions. Let  $\mu_{0,k}$ and  $\Sigma_{0,k}$ denote the prior mean and covariance functions of $f_k, \ k=1,\ldots, K$, respectively. When $g$ is evaluated at $x$, we also get to observe the value of $\f_k$ at $\left(x_{I(k)}, \h_{J(k)}(x)\right)$. Thus, after querying $\obj$ at $n$ points, $x_\ell, \ \ell=1,\ldots, n$, the observations of the values of $\f_k, \ k=1,\ldots, K,$ at $\left(x_{\ell,I(k)}, \h_{J(k)}\left(x_\ell\right)\right)$, $\ell=1,\ldots,n$, imply posterior distributions on $\f_1,\ldots, f_K$, which are again (conditionally) independent GPs with mean and covariance functions $\mu_{n,k}$ and $\Sigma_{n,k}$\footnote{For ease of presentation we assume that all the functions in the network are expensive-to-evaluate, and thus require to be modeled as GPs. However, if any of the functions, say $f_k$, is not expensive-to-evaluate, we can simply take $\mu_k \equiv f_k$ and $\Sigma_k\equiv 0$.}, $k=1,\ldots,K$. These can be computed in closed form using the standard GP regression equations (see, e.g., \citealt{williams2006gaussian}). For completeness, we include these equations in \S\ref{supp:posterior} of the supplement.

The posterior distributions on $f_1,\ldots,f_K$, described above imply a posterior distribution on $g$. This distribution is in general non-Gaussian. Thus, unlike in the standard setting, where $g$ is modeled as a GP, classical acquisition functions such as the expected improvement cannot be computed in closed form. However, as we describe next, drawing samples from this distribution is simple. 

Thanks to the acyclic structure of the underlying network that defines $\obj$, a sample, 
$\widehat{\obj}(x) = \widehat{h}_K(x)$
from the posterior distribution on $\obj$ at $x$ can be obtained 
 following the iterative process described next.
In each iteration, $k=1,\ldots, K$, we 
obtain a sample, $\widehat{h}_k(x)$, from the posterior distribution on $\h_k(x)$
by drawing a sample from the normal distribution with mean $\mu_{n,k}\left(x_{I(k)} , \widehat{h}_{J(k)}(x)\right)$ and standard deviation
\begin{equation*}
\sigma_{n,k}\left(x_{I(k)}, \widehat{h}_{J(k)}\right) = \Sigma_{n,k}\left(x_{I(k)}, \widehat{h}_{J(k)}(x), x_{I(k)}, \widehat{h}_{J(k)}(x)\right)^{1/2}.
\end{equation*}
Supporting efficient calculation,
the samples   
$\widehat{h}_{J(k)}(x)$
will have already been obtained in previous iterations of the for loop since $j<k$ for all $j \in J(k)$ (Note that $J(1)=\emptyset$). This procedure is summarized in Algorithm \ref{alg:sample}.

\begin{algorithm}[h]
\begin{algorithmic}[1]
\caption{Draw one sample from the posterior on $g(x)$}
\label{alg:sample}
\REQUIRE{$x\in\domain$}
\FOR{$k = 1,\ldots,K$}
\STATE{ $\widehat{h}_{k}(x)\sim\mathcal{N}\left(\mu_{n,k}\left(x_{I(k)},\widehat{h}_{J(k)}(x)\right),\sigma_{n,k}\left(x_{I(k)},\widehat{h}_{J(k)}(x)\right)^2\right)$}
\ENDFOR
\RETURN{$\widehat{g}(x)=\widehat{h}_{K}(x)$}
\end{algorithmic}
\end{algorithm}
 
 We end this section by noting that, while the statistical model described above could be catalogued as a deep GP \citep{damianou2013deep}, in the sense that we have GP layers in an architecture described by a directed acyclic graph, inference in our model is faster. Typically, deep GP inference requires marginalization over latent values of GP layers. In our setting, however,  observation structure creates conditional independence across layers, avoiding the need to marginalize.
 
 \subsection{Expected Improvement for Function Networks}
 \label{sec:eifn}
 Our acquisition function is the expected improvement, computed with respect to the implied posterior distribution on $\obj$ under the statistical model described in \S\ref{sec:stat_model}:
 \begin{equation*}
     \eifn_n(x) = \E_n\left[\{\obj(x) - \obj_n^*\}^+\right],
 \end{equation*}
 where $\obj_n^*= \max_{i=1,\ldots, n}\obj\left(x_n\right)$ is the best value observed so far, and $\E_n$ is the expectation computed with respect to the GP time-$n$ posterior distributions on $f_1,\ldots, f_K$. To distinguish it from the classical expected improvement, we refer to our acquisition function as the expected improvement for function networks (EI-FN) in the remainder of this work.

\subsection{Maximization of EI-FN via Sample Average Approximation}
\label{sec:eifn_saa}
Since the posterior distribution on $\obj$ is non-Gaussian, in contrast with the classical expected improvement acquisition function, $\eifn_n$ does not admit a closed form expression. However, $\eifn_n(x)$ can be computed with arbitrary precision following a simple Monte Carlo (MC) approach:
\begin{equation*}
     \eifn_n(x) \approx \frac{1}{M}\sum_{m=1}^{M}\left\{\widehat{\obj}(x)^{(m)}- \obj_n^*\right\}^+,
 \end{equation*}
 where $\widehat{\obj}(x)^{(1)}, \ldots, \widehat{\obj}(x)^{(M)}$ are samples drawn from the  posterior distribution on $\obj(x)$, which can be obtained following the approach described in \S\ref{sec:stat_model}. 
 
 Following the above MC scheme to approximately compute EI-FN, one can aim to maximize this function using a derivative-free global optimization algorithm for inexpensive-to-evaluate functions, such as CMA \citep{hansen2016cma}, by drawing samples, $\widehat{\obj}(x)^{(1)}, \ldots, \widehat{\obj}(x)^{(M)}$ independently for each $x$ at which $\eifn_n$ is evaluated. However, this approach is slow since evaluations of EI-FN are noisy and derivative information is not leveraged. Instead, we propose to maximize EI-FN following a sample average approximation (SAA) approach \citep{kleywegt2002sample, balandat2020botorch}.

%In a nutshell, the SAA approach works by building a single MC approximation that holds across all $x$. This MC approximation is held in place and is optimized over $x$. This allows leveraging derivatives more effectively.

Succinctly, the SAA approach works by building a MC approximation of EI-FN that is deterministic given a finite set of random variables not depending on $x$, and maximizing this approximation instead of EI-FN itself. 
The key observation for creating this approximation is that a sample, $\widehat{\obj}(x)$, can be obtained as $\widehat{\obj}(x) = \widehat{h}_K(x)$, where $\widehat{h}_1(x), \ldots, \widehat{h}_K(x)$ are defined recursively by
\begin{align*}
  \widehat{h}_k\left(x; Z\right) = {} & \mu_{n,k}\left(x_{I(k)} , \widehat{h}_{J(k)}(x;Z)\right) +  \sigma_{n,k} \left(x_{I(k)}, \widehat{h}_{J(k)}(x; Z)\right)Z_k,
\end{align*}
where $Z = \left(Z_1,\ldots, Z_K\right)^\top$ is a standard normal random vector, and we write $\widehat{h}_k(x)$ as $\widehat{h}_k\left(x; Z\right)$ to make its dependence on $Z$ explicit. Analogously, we also write $\widehat{g}(x)$ as $\widehat{g}(x;Z)$. This can be seen as an extension of the so-called reparametrization trick for acquisition functions  \citep{wilson2018maximizing}.

%\vskip 0.3in
%TEMPORARY
%We can produce an unbiased estimator of $\mathrm{EI-FN}_n(x)$ by sampling $g(x)$ from its posterior distribution as described in Section 4.1, calling this sample $\hat{g}(x)$ and then using as our estimate $\{\hat{g}(x) - g^*_n\}^+$. We can then do this many times using independent samples and average the results to get an asymptomitaclly consistent estimator.  This, however, ...

%\begin{align*}
  %\widehat{h}_k\left(x\mid Z\right) = {} & \mu_{k}^{(n)}\left(x_{I(k)} , \widehat{h}_{J(k)}(x\mid Z)\right) + 
  %\sigma_k^{(n)} \left(x_{I(k)}, \widehat{h}_{J(k)}(x\mid Z)\right)Z_k,
%\end{align*}
%\begin{equation*}
%\label{eq:ei_mc2}
    %\widehat{\eifn}_n\left(x \mid Z^{(1:M)}\right) := \frac{1}{M}\sum_{m=1}^{M}\left\{\widehat{\obj}\left(x\mid Z = Z^{(m)}\right) - \obj_n^*\right\}^+
 %\end{equation*}
%\vskip 0.3in

If we now fix $M$ samples drawn from the $K$-dimensional standard normal distribution, $Z^{(1)},\ldots, Z^{(M)}$, then
\begin{equation*}
%\label{eq:ei_mc2}
    \widehat{\eifn}_n\left(x ; Z^{(1:M)}\right) := \frac{1}{M}\sum_{m=1}^{M}\left\{\widehat{\obj}\left(x; Z^{(m)}\right) - \obj_n^*\right\}^+
 \end{equation*}
 is a MC approximation of EI-FN that is deterministic given $Z^{(1)},\ldots, Z^{(M)}$, as desired. Moreover, Proposition \ref{prop:saa} below shows that, under mild regularity conditions, any maximizer of the above SAA converges in probability exponentially fast to a maximizer of $\eifn_n$ as $M\rightarrow\infty$, thus suggesting that in practice it is safe to use small values of $M$. This result is a  generalization of Theorem 1 in \cite{balandat2020botorch}. Its proof can be found in \S\ref{supp:prop1} of the supplement.
 
 \begin{proposition} 
 \label{prop:saa}
 Suppose that the functions $\mu_{n,k}$ and $\sigma_{n,k}$, $k=1,\ldots, K$, are all Lipschitz continuous and let
 \begin{equation*}
     \widehat{x}_{*}^{(M)}\in\argmax_{x\in\domain} \widehat{\eifn}_n\left(x ; Z^{(1:M)}\right),
     \quad
     X_* =  \argmax_{x\in\domain}\eifn_n(x).
 \end{equation*}
 Then, for every $\epsilon>0$, there exist $A, \alpha > 0$ such that $\prob\left(\textnormal{dist}\left(\widehat{x}_{*}^{(M)}, X_*\right)>\epsilon\right) \leq Ae^{-\alpha M}, \ M\in\mathbb{N}$.
 \end{proposition}
 
  Finally, we note that $\widehat{\eifn}_n$ is differentiable with respect to $x$, provided that $\mu_{k,n}$ and $\Sigma_{k,n}$, $k=1,\ldots, K$, are all differentiable. Thus, $\widehat{\eifn}_n$ can be maximized using a gradient-based deterministic optimization algorithm such as L-BFGS-B \citep{byrd1995limited}, with multiple restarts.
 \subsection{Asymptotic Consistency of EI-FN without Dense Measurements}
 \label{sec:ac}
 To shed light on the convergence properties of the expected improvement acquisition function under our statistical model, we prove that, under suitable regularity conditions, EI-FN is asymptotically consistent, meaning that it finds the global optimum of the objective function as the number of evaluations grows to infinity. 
 This builds on results shown for the classical expected improvement \citep{vazquez2010convergence,bect2019supermartingale} and the expected improvement for composite functions \citep{astudillo2019bayesian}, and we rely on similar assumptions.
This result is stated in Theorem \ref{thm:ac} below and its proof can be found in \S\ref{supp:thm1} of the supplement.
 \begin{theorem}
 \label{thm:ac}
Suppose that $x_{n+1}\in\argmax_{x\in\domain}\eifn_n(x)$ for all $n$. Then, under regularity conditions stated precisely in the supplement, 
% the following two statements hold.
% \begin{enumerate}
% \item $g_n^*\rightarrow \max_{x\in\domain}g(x)$ as $n\rightarrow\infty$ provided that $f_k$ lies within the reproducing kernel Hilbert space associated with $\Sigma_{0,k}$ for $k=1,\ldots, K$. 
    $g_n^*\rightarrow \max_{x\in\domain}g(x)$ as $n\rightarrow\infty$ almost surely under the prior distribution on $f_1,\ldots, f_K$. 
% \end{enumerate}
 \end{theorem}
 
Significant novelty in our proof arises from the fact that EI-FN's measurements are not necessarily dense in $\mathbb{X}$. 
In nearly all existing work, consistency of BO methods is shown by first showing that the measurements are dense in the objective's domain (see, e.g., \citealt{vazquez2010convergence}).
Surprisingly, the function network structure may allows us to exclude certain regions of $\mathbb{X}$ as suboptimal after only finitely many measurements, allowing our method to be consistent without measuring everywhere. Such ability gives insight into EI-FN's strong empirical performance compared to methods ignoring function network structure.
As stated in Proposition~\ref{sec:not-dense} below, we provide a function network where EI-FN provides a consistent estimate of the global optimum {\it without} sampling densely.

\begin{proposition}
\label{sec:not-dense}
There exists a function network (detailed in \S\ref{supp:prop2} of the supplement) in which EI-FN is consistent but whose measurements are not necessarily dense in $\domain$.
\end{proposition}

% The main novelty in our proof is coping with the randomness in node input, the fact that nodes can depend on only a subset of the decision variables. Remarkably, these properties of our setting imply that EI-FN may not necessarily sample the feasible space densely despite being asymptotically consistent. 
% The main novelty in our proof is coping with the randomness in node input and the fact that nodes can depend on only a subset of control inputs. The latter makes it possible for a node to be repeatedly evaluated with the same inputs, even though we show that no $x$ is evaluated twice. 
 
% We end this section by noting that existing convergence rates for the classical expected improvement heavily rely on properties of its analytical expression \citep{bull2011convergence,ryzhov2016convergence}, and thus are not directly generalizable to our setting. This is, however, an exciting direction for future work. 
 
 \section{Numerical Experiments}
 \label{sec:experiments}
We compare the performance of our algorithm (EI-FN) against the classical expected improvement (EI), i.e., the expected improvement under a GP model over $\obj$. We also compare with the performance of two other standard algorithms: the  algorithm that chooses the points to evaluate uniformly at random over $\domain$ (Random); and the knowledge gradient (KG) (also under a GP model over $\obj$), a more sophisticated acquisition function that often delivers a better performance than the expected improvement \citep{wu2016parallel}. 
The problems in \S\ref{sec:fetch} and \S\ref{sec:calibration}  fall within the framework of \cite{astudillo2019bayesian} and we include its proposed EI-CF algorithm as a benchmark. 
%., which is designed for problems that are a composition of a single vector-valued time-consuming function and an inexpensive 
All algorithms were implemented in BoTorch \citep{balandat2020botorch}.

We evaluate on 4 synthetic problems and 5 real-world problems. These are described below or in \S\ref{supp:details_numerical} of the supplement, with the supplement describing a manufacturing problem building on \S\ref{sec:intro}, a COVID-19 testing problem building on \S\ref{sec:calibration}, and a robot control problem similar to the one described in \S\ref{sec:fetch}.
In all problems, a first stage of evaluations is performed using $2(d+1)$ points chosen uniformly at random over $\domain$. A second stage (pictured in plots) is then
performed using each of the algorithms. Error bars in Figure~\ref{fig:results} show
the mean of  the best objective value observed so far, plus
and minus 1.96 times the standard deviation divided by the square root of the number of replications. Since the difference in performance in some of our experiments is better appreciated in a logarithmic scale, we also include plots showing the log$_{10}$-regret for such experiments in Figure~\ref{fig:results_lr}.  Each experiment was replicated 30 times. Experimental setup details and runtimes are available in the supplement. Code to reproduce our numerical experiments can be found at \url{https://github.com/RaulAstudillo06/BOFN}.
% Expelimental setup details to be moved to supplement:
% How we're estimating hyperparameters
% Total amount of compute, and type of resources
% Licenses of any assets
% All GPs in our experiments  have a constant mean function and ARD Mat\'ern covariance function with smoothness parameter equal to 5/2, which is a standard choice in practice.
% We use $M=128$ samples when maximizing EI-FN and EI-CF via the SAA approach described above. 

%We first describe all problems and then discuss results.

\subsection{Synthetic Test Functions}
\label{sec:synthetic}

We create  synthetic test problems by arranging standard test functions from the global optimization literature \citep{jamil2013literature} into function networks. These explore a variety of network structures in an easy-to-reproduce form, and are named after the standard test function used to define the function network. We describe these briefly here and then in full detail in \S\ref{supp:details_numerical} of the supplement.

\textbf{Alpine2} and \textbf{Rosenbrock} both arrange $K$ nodes in series, where each node except the first node takes the output of the previous node as input. 
Additionally, in Alpine2, each node takes a distinct dimension of the decision vector $x$ as input.
In Rosenbrock, $x_1$ and $x_2$ are inputs to the first node, $x_2$ and $x_3$ are inputs to the second, and so on.
These network architectures arise in manufacturing problems like the example in \S\ref{sec:intro}, as well as business operations with queues like boarding an aircraft or fulfilling drive-through orders. For Alpine2 we set $K=6$, and for Rosenbrock we set $K=4.$ %We consider $K=2, 4, 6$ nodes. The plots for the cases $K=2,4$, and $K=2,6$ for Alpine2 and Rosenbrock, respectively, are deferred to the supplement.

\textbf{Ackley} has 3 nodes. The first two nodes each take the same 6-dimensional input. Their outputs are passed to the third node that produces the final output.  This type of function network arises in algorithm design for two-sided markets \citep{li2021interference}, like Uber and AirBnB, where the first node simulates an intervention's effect on riders (or guests), the second simulates its effect on drivers (or hosts), and the third simulates the matching process where riders and drivers (or guests and hosts) interact to produce transactions.

\textbf{Drop-Wave} has two nodes.  The first node takes a two-dimensional vector $x$ as input.  This node's output is passed to the second node, which produces the objective value.
This network architecture is representative of multidisciplinary engineering design \citep{benaouali2019multidisciplinary}, for example in aerospace, where a small number of distinct black-box simulators simulate processes governed by physical laws that affect each other through a small number of channels, such as an aircraft engine simulation (the first node) determining heat generated while flying, which is then inputted to a temperature-dependent simulation of mechanical stress on the aircraft's frame (the second node).

\begin{figure}
\centering
\begin{tabular}[b]{c}%
\includegraphics[width=0.24\textwidth]{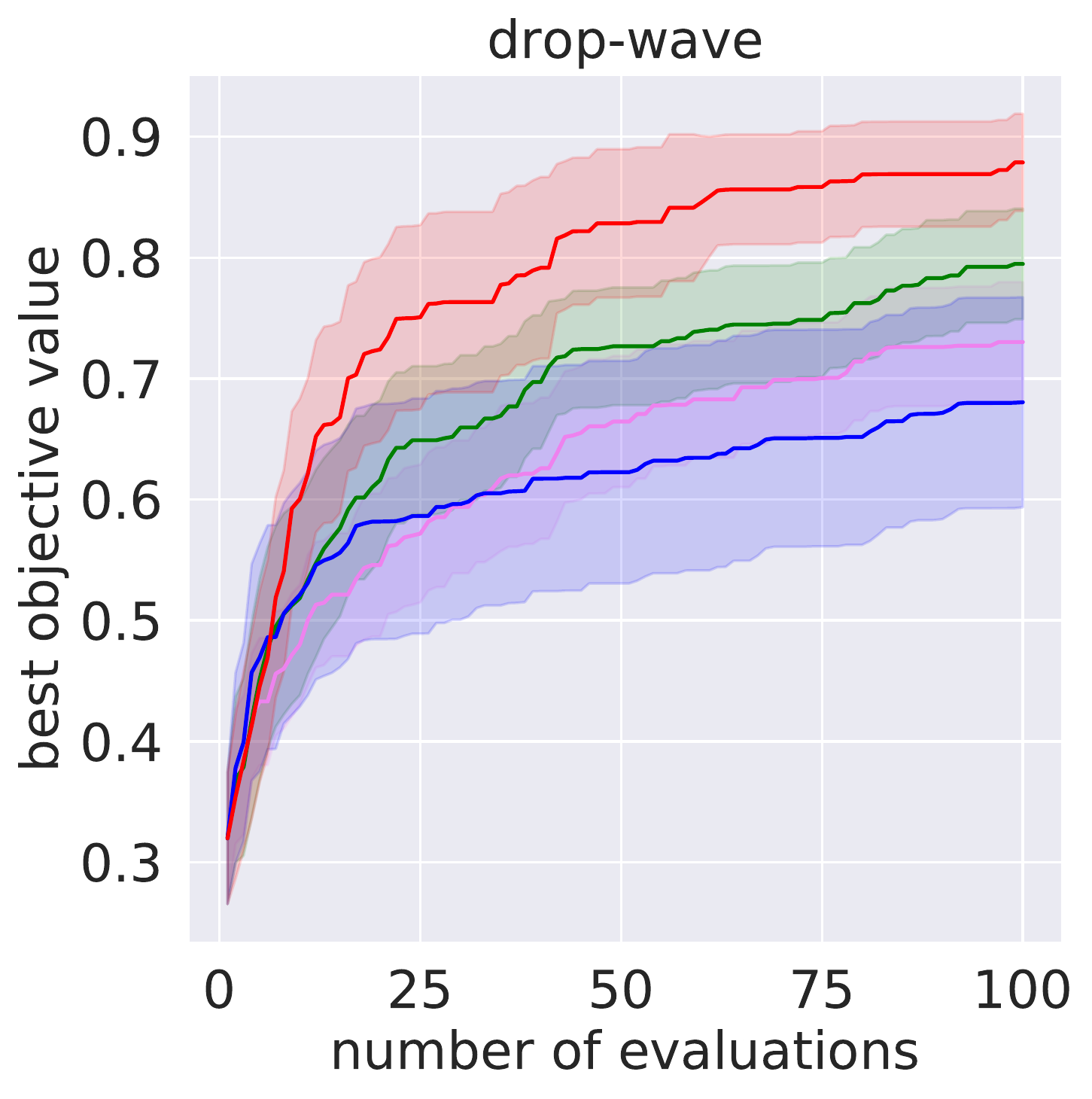}
  \includegraphics[width=0.25\textwidth]{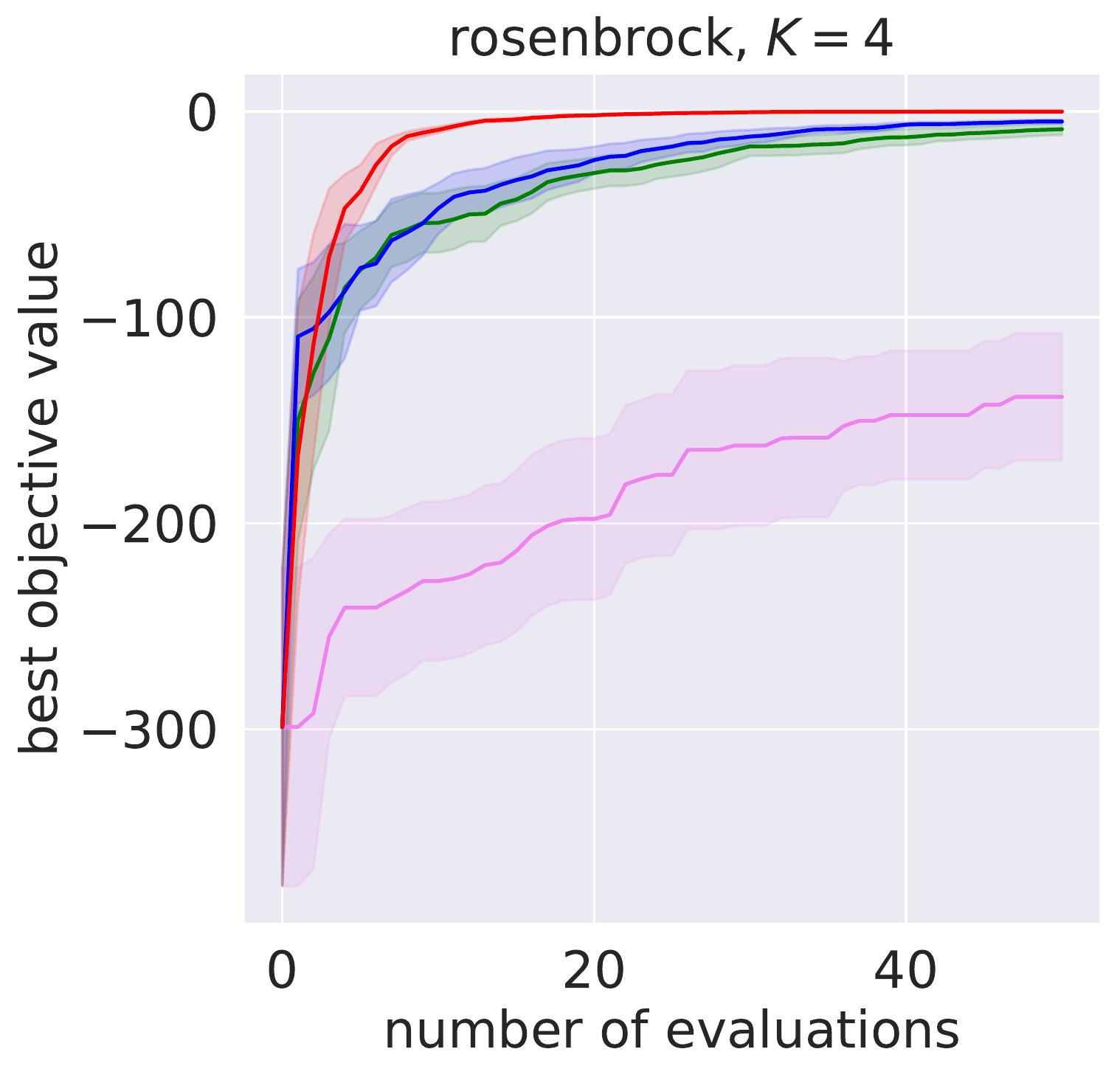}
  \includegraphics[width=0.24\textwidth]{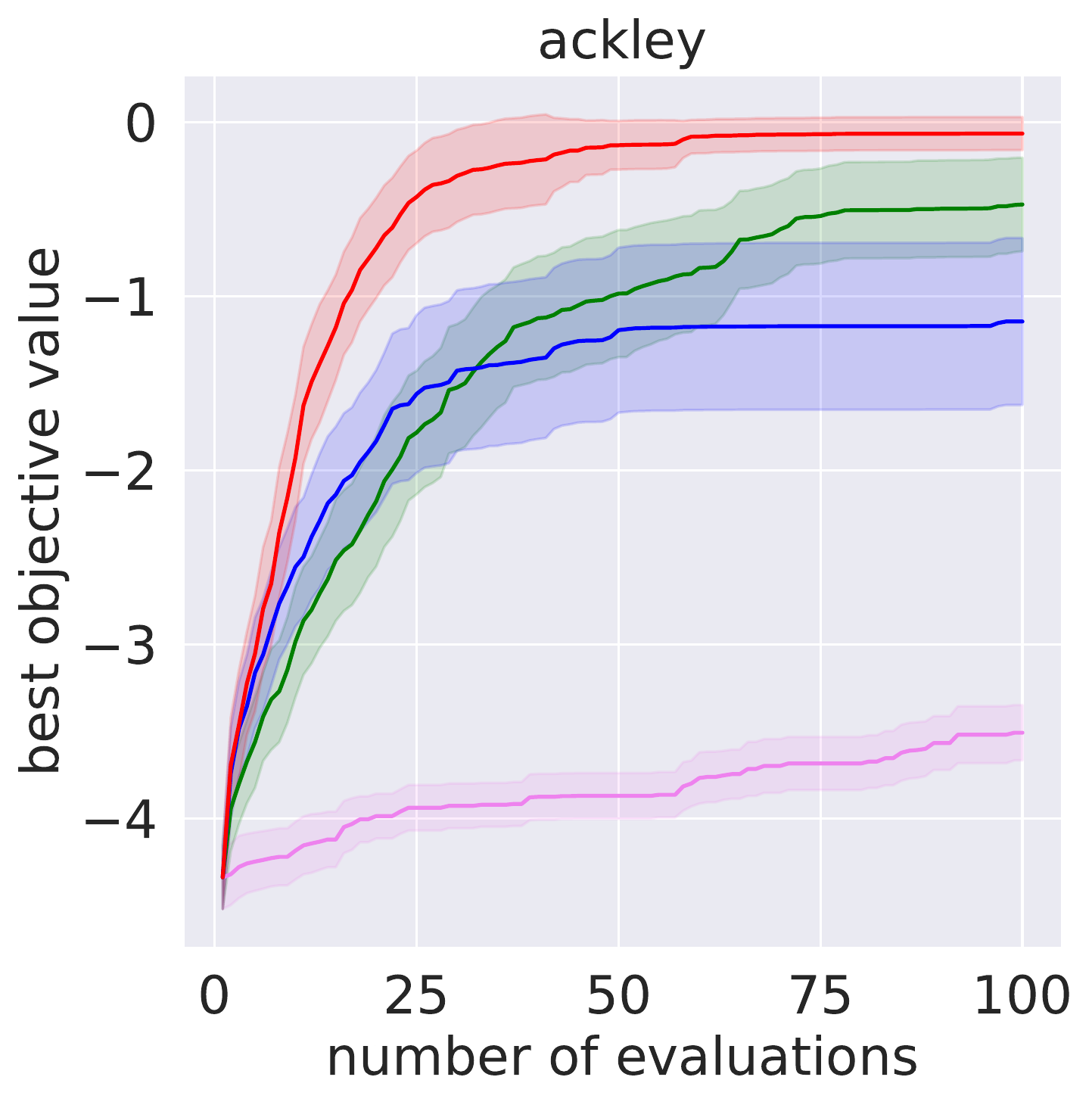}
  \includegraphics[width=0.245\textwidth]{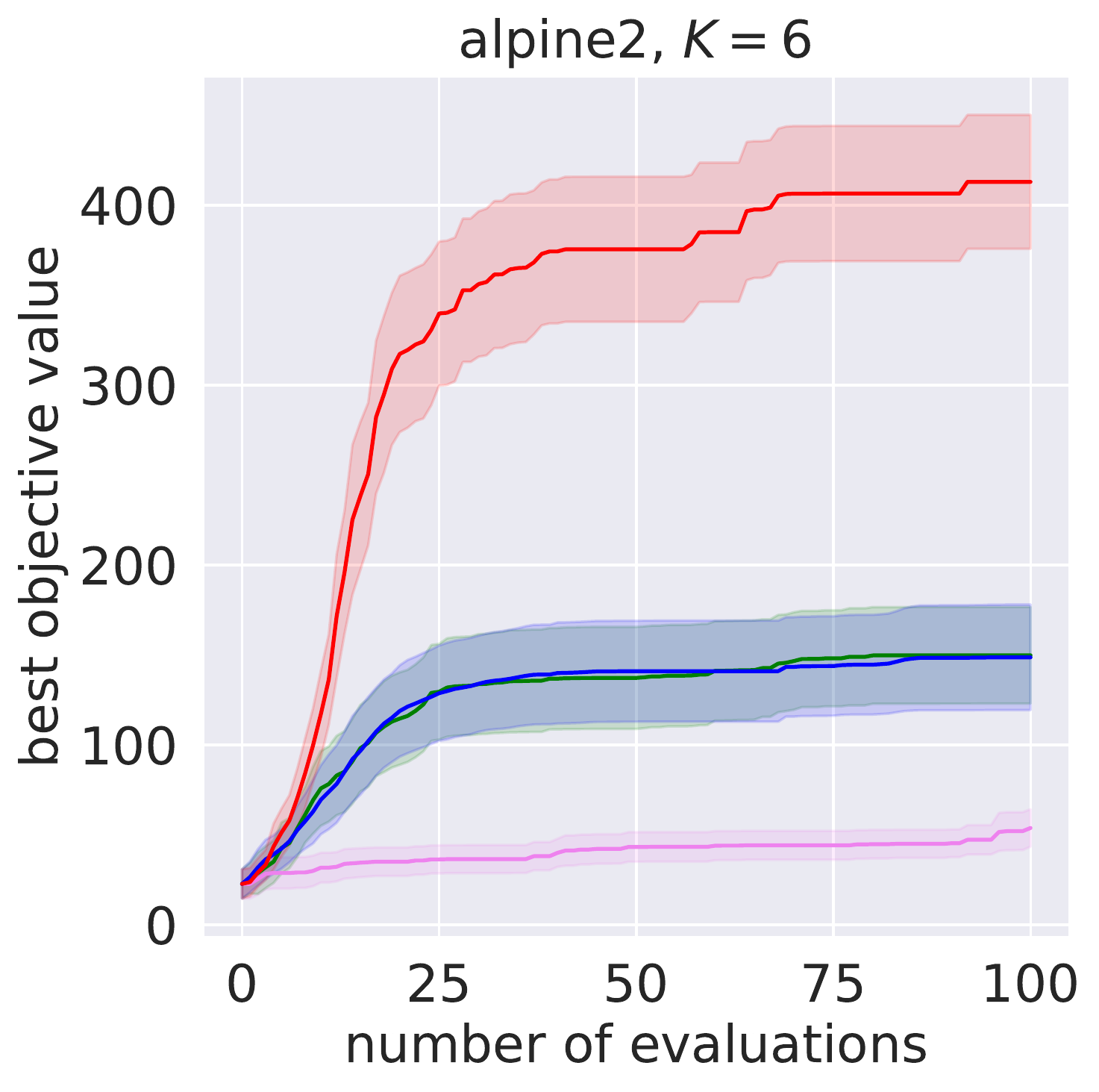}\\
\includegraphics[width=0.235\textwidth]{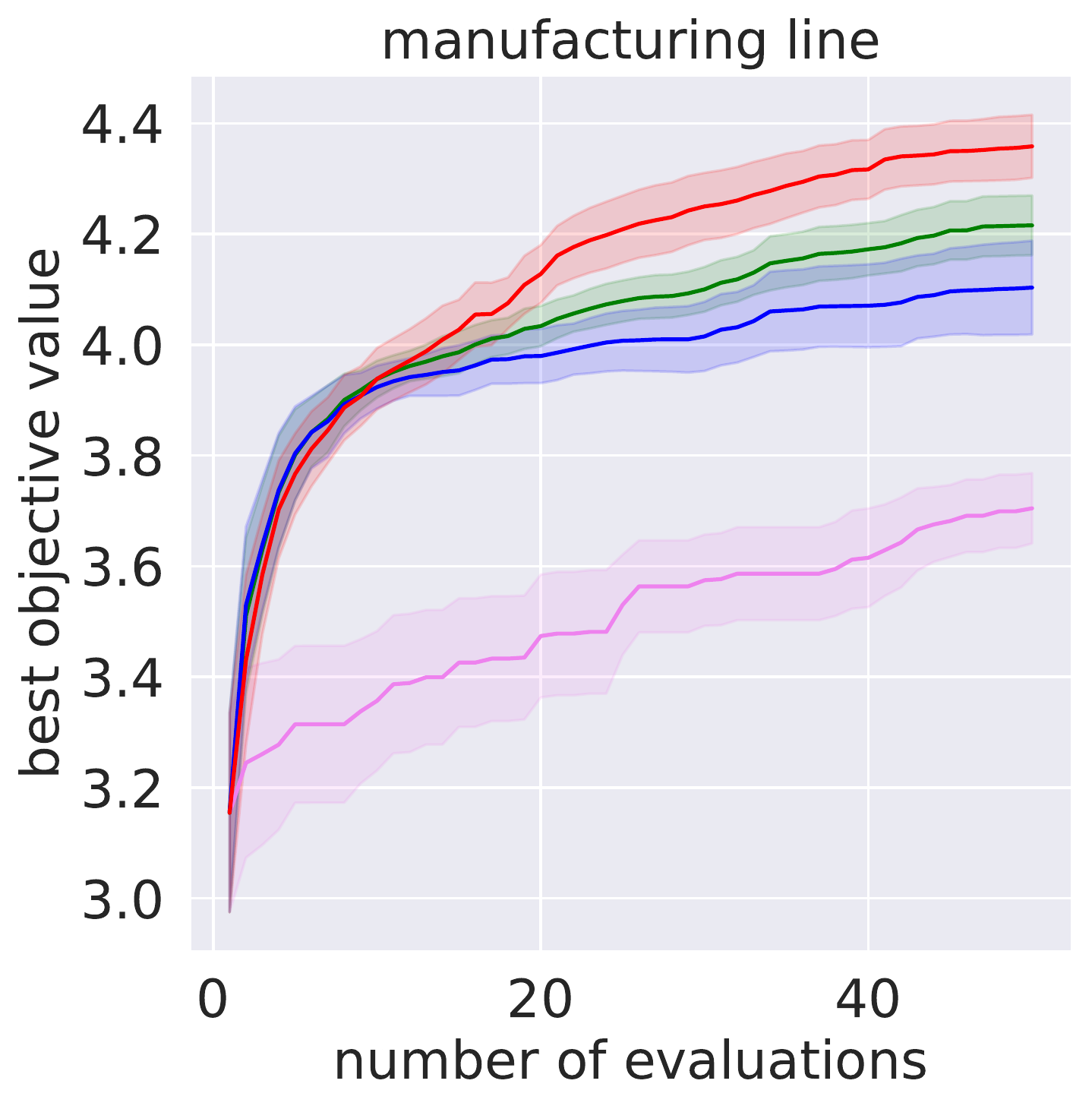}
  \includegraphics[width=0.255\textwidth]{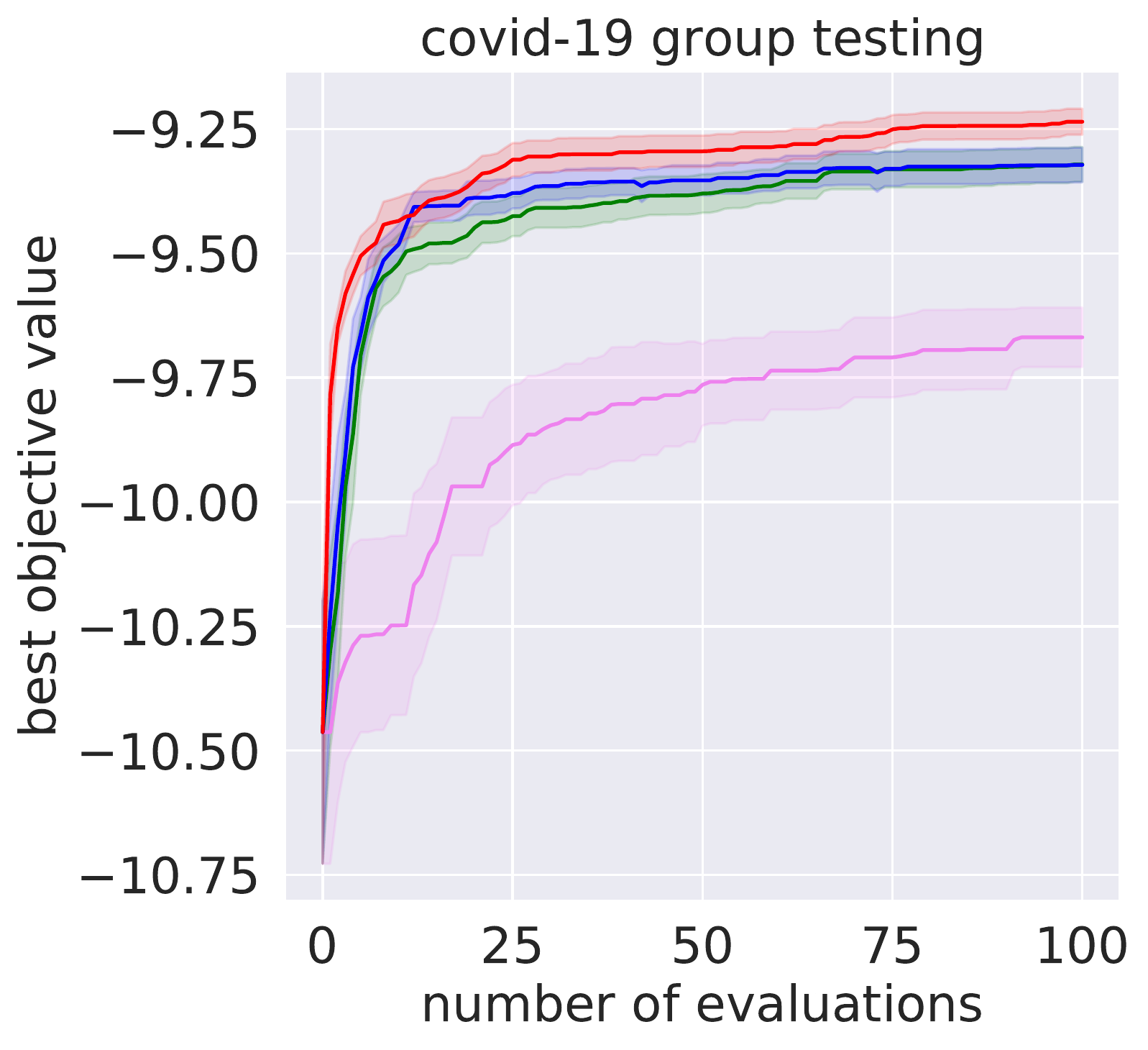}
  \includegraphics[width=0.24\textwidth]{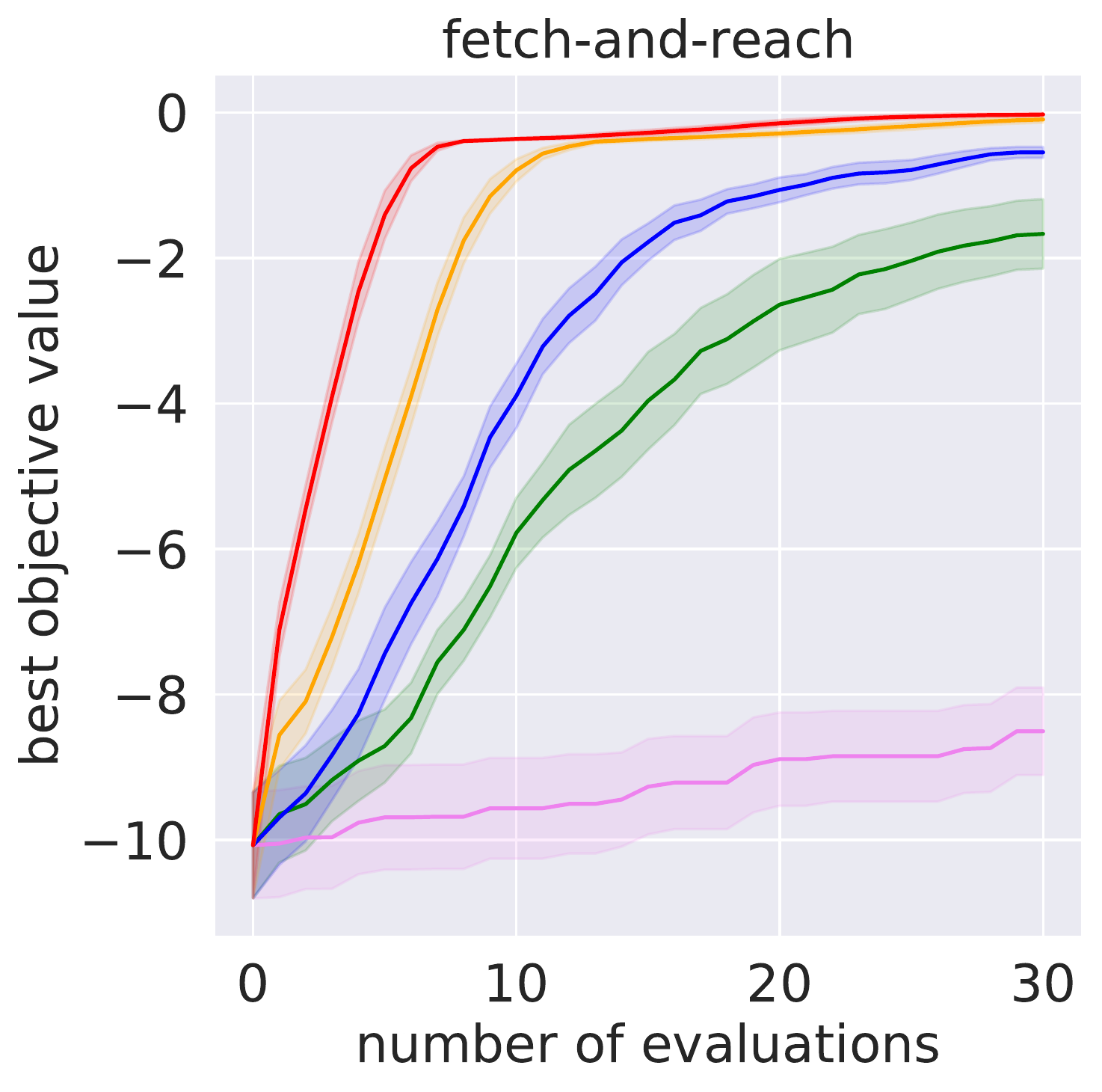}
  \includegraphics[width=0.24\textwidth]{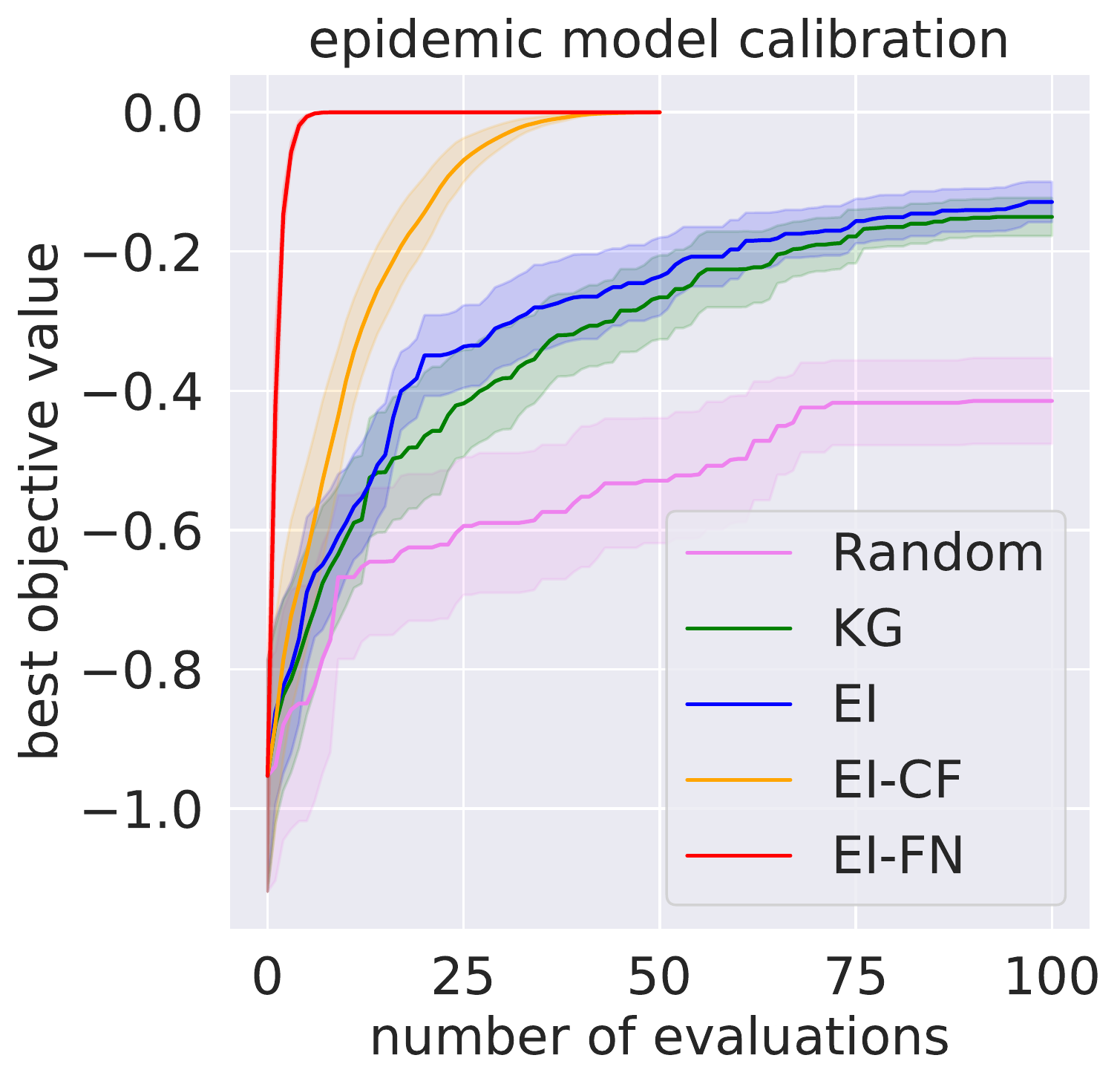}
 \end{tabular}
 \caption{Top:
 Results on synthetic problems that adapt widely used synthetic test functions into function networks.
Bottom: Results on realistic problems: manufacturing line, design of testing protocols for COVID-19, fetch-and-reach with a robotic arm, and  calibration of an epidemic model. EI-FN substantially improves over benchmark methods, with larger improvements for problems with higher-dimensional decision vectors and more nodes.
 % The top left, middle and right use the Alpine2 function to create function networks with $K=2,4,6$ nodes with $D=1,3,5$-dimensional controls respectively.
 \label{fig:results}}
\end{figure}

\begin{figure}
\centering
\begin{tabular}[b]{c}%
  \includegraphics[width=0.243\textwidth]{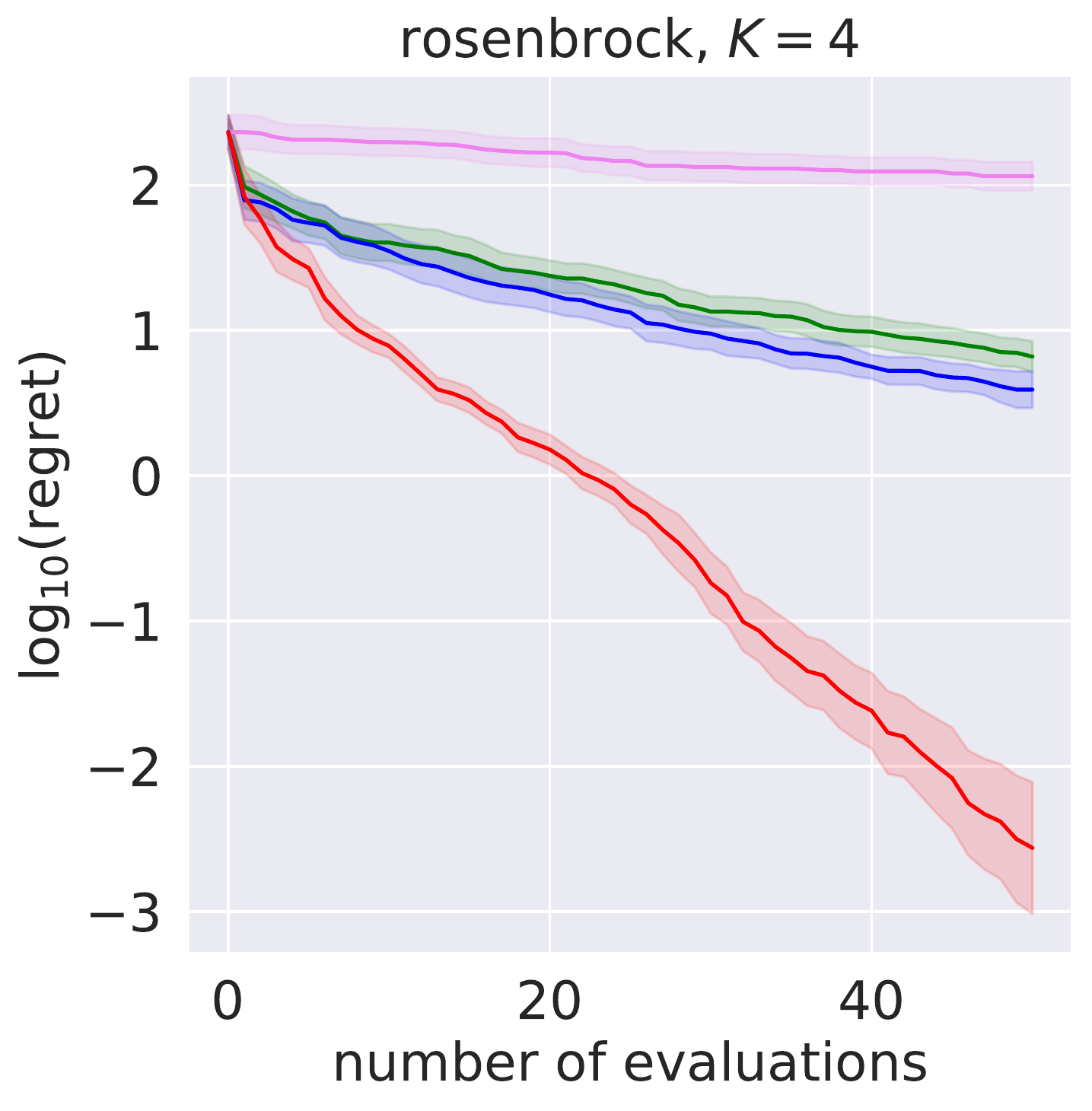}
  \includegraphics[width=0.25\textwidth]{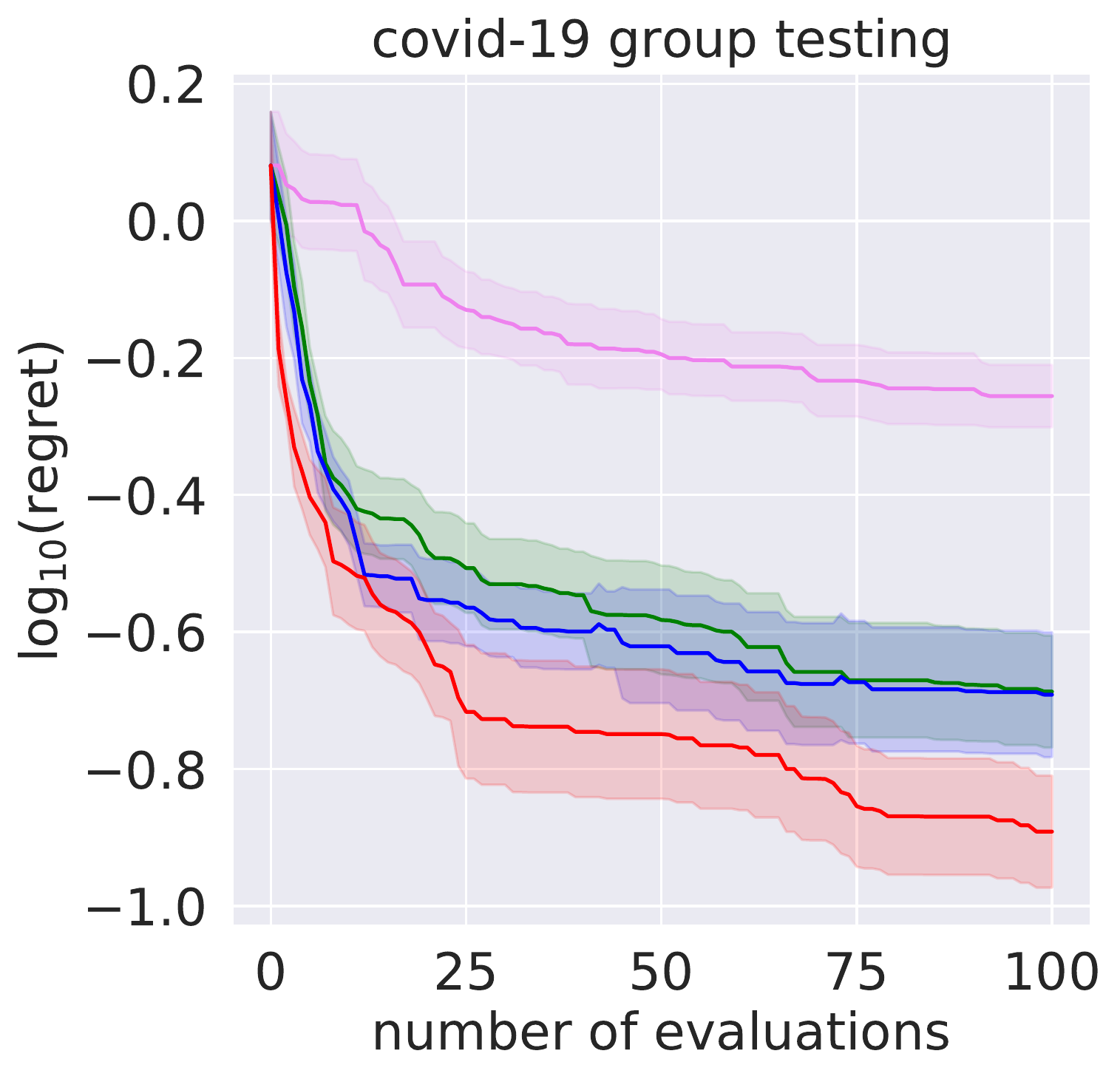}
  \includegraphics[width=0.25\textwidth]{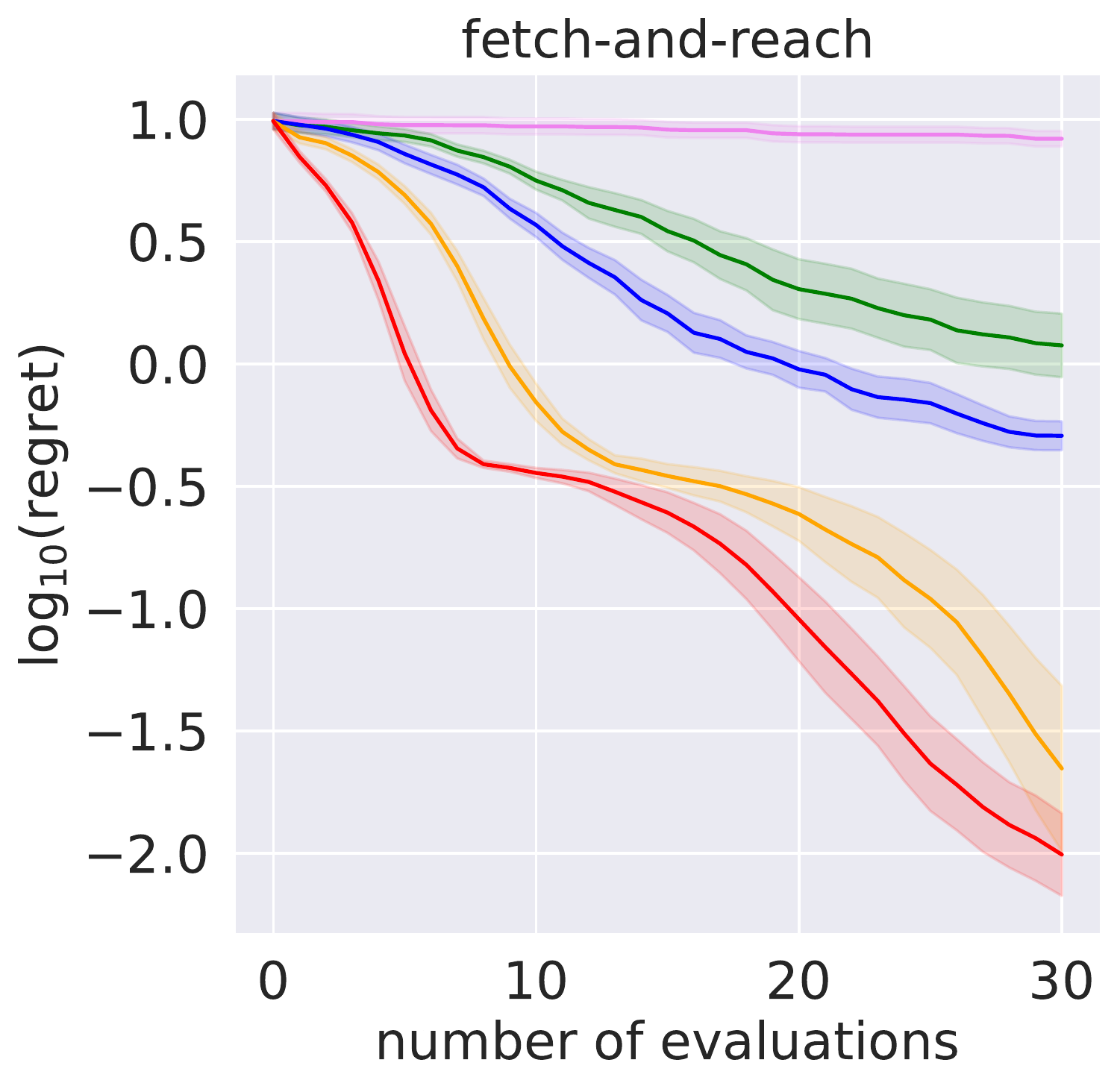}
  \includegraphics[width=0.243\textwidth]{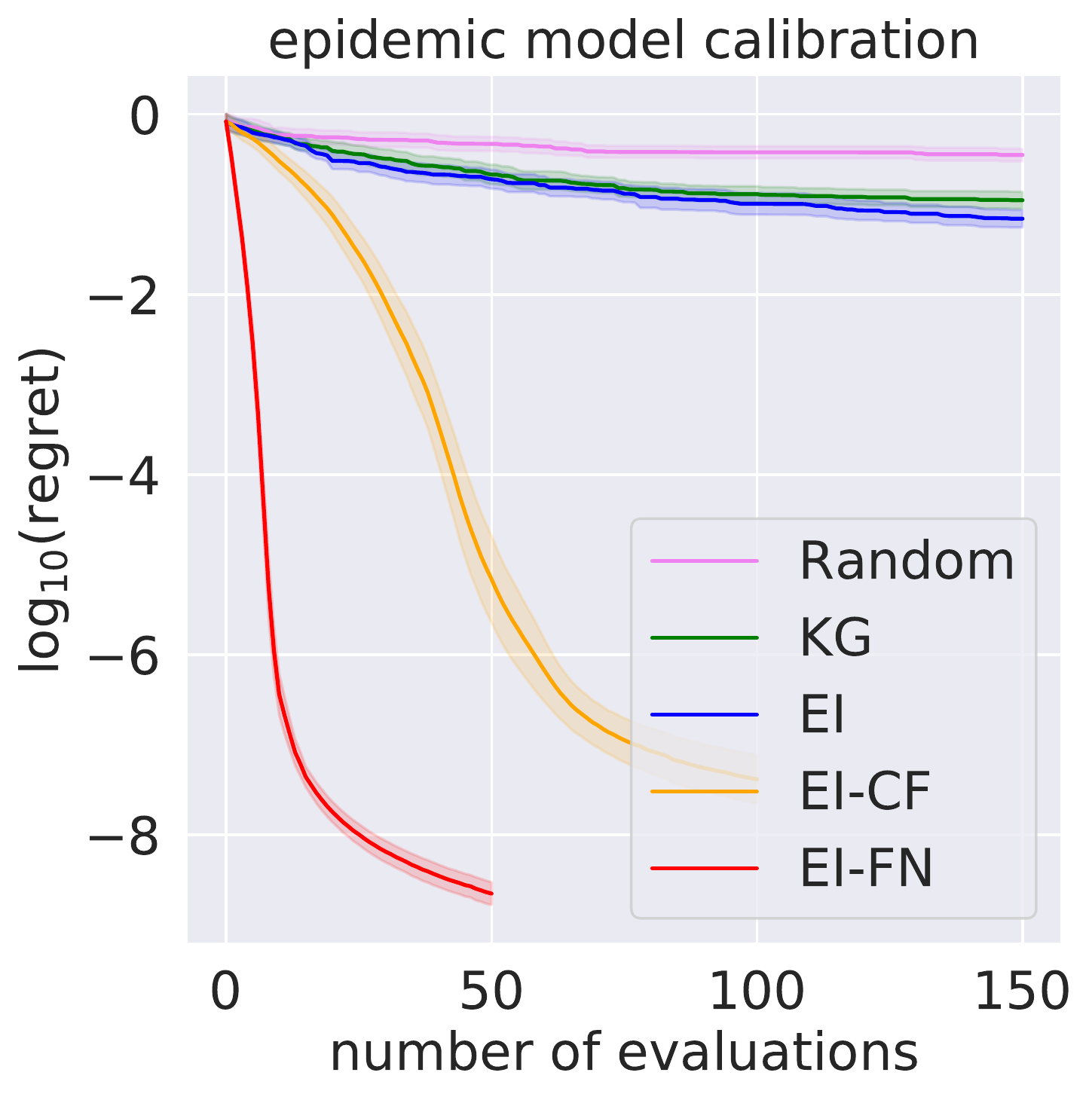}
 \end{tabular}
 \caption{ Results on four of our test problems. In contrast with Figure~\ref{fig:results} above, which shows the best objective value found, here we plot the log$_{10}$-regret. 
 \label{fig:results_lr}}
\end{figure}

%\begin{figure}[h]
  %\centering
  %\includegraphics[width=0.48\textwidth]{figures/alpine2_2.pdf}
  %\includegraphics[width=0.48\textwidth]{figures/alpine2_4.pdf}
  %\includegraphics[width=0.48\textwidth]{figures/alpine2_6.pdf}
  %\caption{Average performance over 30 replications on the Alpine2 test function for 2, 4, and 6 nodes (from left to right). EI-FN  outperforms other methods by a large margin, which increases with the number of nodes. \label{fig:alpine2}}
%\end{figure}

%\begin{figure}[h]
  %\centering
  %\includegraphics[width=0.48\textwidth]{figures/rosenbrock_3.pdf}
  %\includegraphics[width=0.48\textwidth]{figures/rosenbrock_5.pdf}
  %\includegraphics[width=0.48\textwidth]{figures/rosenbrock_7.pdf}
  %\caption{Average performance over 30 replications on the Alpine2 test function for 2, 4, and 6 nodes (from left to right). EI-FN  outperforms other methods by a large margin, which increases with the number of nodes. \label{fig:alpine2}}
%\end{figure}

%\begin{figure}[h]
%  \centering
%  \includegraphics[width=0.48\textwidth]{figures/ackley_6.pdf}
%  \caption{Average performance over 30 replications on the Ackley1 test function. \label{fig:ackley}}
%\end{figure}

%\begin{figure}[h]
%  \centering
%  \includegraphics[width=0.48\textwidth]{figures/dropwave.pdf}
%  \caption{Average performance over 30 replications on the Drop-Wave test function.}
%\end{figure}

\subsection{Fetch-and-Reach with a Robotic Arm}
\label{sec:fetch}
This test problem is obtained by adapting the Fetch environment from OpenAI Gym (see \citet{plappert2018multi}). The goal is to move the gripper of a robotic arm to a target location with only three movements. We formulated this problem as a function network with 3 nodes, each representing a movement of the robotic arm. Each of these nodes takes as input the current  location of the gripper along with a vector of forces to be applied to the robotic arm in that step, and produces as output the location of the arm after this movement is complete. (Note that the output of each node is 3-dimensional and thus this can also be thought of as a function network with 9 single-output nodes). The objective to minimize is the distance between the gripper and the target in the final step. Figure~\ref{fig:fetch_anim} shows an animation of this problem.

We formalize the above problem as follows. Let $z_\mathrm{init}, z_\mathrm{target}\in\R^3$ denote the object's initial and target locations, respectively. At each time step, $t$, we choose the vector of forces to be applied to the robotic arm $x_t\in[-1,1]^3$. After this movement, the location of the object becomes $z_{t+1}$. The goal is to choose $x_t$ for $t=1,\ldots, T$ to minimize $\|z_\mathrm{target} - z_T\|_2$. We set $z_\mathrm{init}=(0,0,0)$, $z_\mathrm{target}=(-12,13,0.2)$, and $T=3$. This can be interpreted as a function network by associating each time step with a triplet of node functions $f_t = (f_{t,1}, f_{t,2}, f_{t,3})$  which take $x_t$ as input and produce $x_{t+1} = f_t(x_t)$ as output.

A very similar experiment to the one described above can be found in \S\ref{supp:robot_push} of the supplement. It considers a a variation of the active learning for robot pushing problem introduced by \cite{wang2017max} whose goal is to teach a robot to push an object to a predetermined target location. As in the experiments here, EI-FN outperforms other methods significantly, including EI-CF.

\begin{figure}
  \centering
  \includegraphics[width=0.27\textwidth]{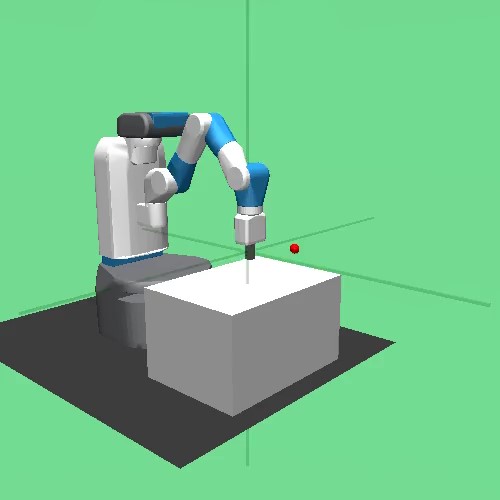} 
  \includegraphics[width=0.27\textwidth]{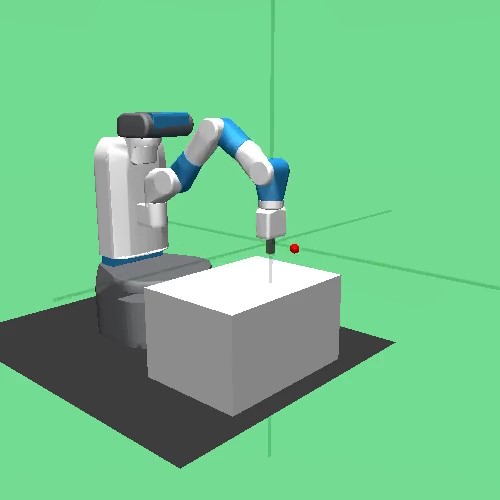}
  \includegraphics[width=0.27\textwidth]{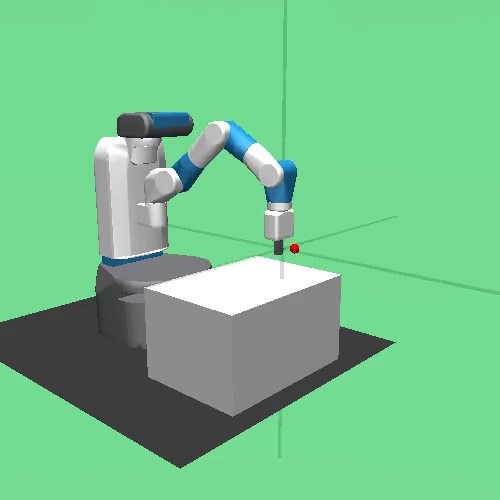}
  \caption{A sequence of screenshots showing three consecutive movements performed by the robotic arm described in \S\ref{sec:fetch}.\label{fig:fetch_anim}}
\end{figure}

\subsection{Calibration of an Epidemic Model}
\label{sec:calibration}
% {\it PF: intro should be revised. Let's point out that these problems help demonstrate the utility of function networks.}
Here we consider calibration of compartmental stochastic models to data, a widely-used tool in epidemiology, medical modeling, and ecology   \citep{sandberg1978mathematical}.
Function networks are well-suited to exploit the structure of such models.
We focus on calibration of a specific epidemic model for influenza, building toward a COVID-19 mitigation benchmark in the next section.
We first describe the epidemic model, then the calibration problem and, finally, formulation as a function network.

\newcommand{\optional}[1]{}

\textbf{SIS Epidemiological Model:}
We calibrate to data a widely used epidemiological model, the SIS model (see, e.g., \citealt{garnett2002introduction}),
that models diseases like influenza capable of reinfecting individuals multiple times. 
In this model, individuals either do not have the disease and are ``susceptible'' (S) or have the disease and are ``infectious'' (I).
We consider a SIS model of two interacting populations, where 
% (e.g., those older than the age of 65, and those younger than 65), where 
% \optional{This models, for example, people living in different parts of the world, which is understood to be important for understanding seasonality in influenza \citep{zhang2019predicting}, or people of different age groups since social contact is assortative by age  \citep{mossong2008social}.}
infections occur at 
population-dependent rates.
% \optional{ is heterogeneous, with age groups having a greater fraction of their contact with other people of the same approximate age, and younger people having more contact overall \citep{mossong2008social}.
% This can even model an infectious disease such as malaria carried by an animal population \citep{bacaer2011ross}. }
% , malaria does 
% This model has also been used to model the spread of malaria  \citep{}, with one population modeling mosquitos (which carry malaria) and the other modeling humans.
% Such heterogeneity in interactions are known to be important for modeling real populations (CITE).

The model is dynamic, with time indexed by $t=0,1\ldots,T$.
At the beginning of each time period $t$, the fraction of population $i \in \{0,1\}$ that is infectious is $\I{i}{t}$. We assume each population is of equal size, $N$.
During this time period, each person in group $i$ comes into close physical contact with $\beta_{i,j,t}$ people from group $j$.
When this contact is between an infectious person and a susceptible one, it infects the susceptible person.
A fraction $\I{j}{t}$ of the people from group $j$ involved in such interactions are infectious and a fraction $(1-\I{i}{t})$ from group $i$ are susceptible.
A number of new infections in group $i$ result, 
$N (1-\I{i}{t}) \beta_{i,j,t} \I{j}{t}$.
As a fraction of group $i$'s population, this is $(1-\I{i}{t}) \beta_{i,j,t} \I{j}{t}$.
Summing across $j$, we have $(1-\I{i}{t}) \sum_j \beta_{i,j}(t) \I{j}{t}$ new group $i$ infections.
At the same time, infections resolve at a rate of $\gamma$ per period. This results in a decrease in the infectious population in group $i$ of $\gamma \I{i}{t}$.
Putting this together, the number of infectious individuals in group $i$ at the start of the next time period is 
%\begin{equation*}
$\I{i}{t+1} = \I{i}{t} (1-\gamma) + (1-\I{i}{t}) \sum_j \beta_{i,j,t} \I{j}{t}$.
%\end{equation*}

\textbf{Calibration:}
The SIS model has parameters $\I{i}{0}$, $\gamma$ and $\beta_{i,j,t}$, where $0\le t < T$ and $i,j \in \{1,2\}$.
We calibrate the parameters $\vec{\beta} = (\beta_{i,j,t} : i,j,t)$ while fixing $\gamma=0.5$ and $\I{i}{0}=.01$ (for both $i$). 
\optional{Such transmission parameters vary substantially over seasons for the influenza and common cold \citep{fuhrmann2010effects}, due to seasonal changes in behavior patterns, temperature and humidity.} We simulate a trajectory of infections from $t=0$ up to $T=3$, using a held-out value for $\vec{\beta}$. We let $I^{\mathrm{obs}}_{i,t}$ denote the fraction of group $i$ observed to be infected at time $t$ in this trajectory.
We then search for the vector $\vec{\beta}$ that, when passed to the SIS model, minimizes the mean squared error between this trajectory and the SIS model predictions. Letting $\I{i}{t}(\vec{\beta})$ indicate this predicted value, the goal is to minimize the mean-squared error (MSE),
% \begin{equation*}
%\min_{\vec{\beta}} 
$\mathrm{MSE}(\vec{\beta}) := \sum_{t=1}^T (I^{\mathrm{obs}}_{i,t} - \I{i}{t}(\vec{\beta}))^2$.
% \end{equation*}
We do not include $t=0$ since $\I{i}{0}(\vec{\beta})$ is the same for all $\vec{\beta}$.

\textbf{Formulation as a Function Network:}
We encode this as a function network using $2T+1$ nodes, as illustrated in Figure~\ref{fig:epidemic-covid}. 
For each time period $t$ and each group $i$, a node takes input $I_t := (\I{j}{t} : j=0,1)$ and $\beta_t := (\beta_{j,j',t}:j,j'\in\{0,1\})$ and produces output $\I{i}{t+1}$.  (For $t=0$, the input $I_0$ is not needed since this is the same for all $\vec{\beta}$.) 
Then, one additional node takes the output of the other nodes $(\I{i}{t} : i=0,1, t=1,\ldots,T)$ as its input and produces the sum of squared errors 
$\sum_{t=1}^T (I^{\mathrm{obs}}_{i,t} - \I{i}{t}(\vec{\beta}))^2$ as output. We treat this final node as known (its GP prior has a kernel of 0).

\textbf{EI-CF benchmark:}
The fact that the final node in this problem (denoted ``MSE'' in Figure~\ref{fig:epidemic-covid}) has known structurre permits comparing against the EI-CF method for BO of composite functions \citep{astudillo2019bayesian} as a benchmark. EI-CF is substantially less general than our method (EI-FN): it is restricted to settings with one time-consuming black-box multi-output node that provides input to one fast-to-evaluate node with known structure.
To apply EI-CF to this problem, the single black-box multi-output node takes $\vec{\beta}$ as input and produces the vector $(I_{i,t}(\vec{\beta}) : i,t)$ as output. This output is then supplied to the ``MSE'' node. This approach ignores the fact that $I_{i,t}$ does not depend on $\beta_{t'}$, $t'>t$, and depends only indirectly on $\beta_{t'}$, $t<t$ through $I_{j,t-1}$, $j=1,2$.

\begin{figure}
  \centering
 \begin{tikzpicture}[
init/.style={
  draw,
  circle,
  inner sep=0.7pt,
  minimum size=0.7cm
},
init2/.style={
  circle,
  inner sep=0.7pt,
  minimum size=0.7cm
},
]
\begin{scope}[start chain=1,node distance=8mm]
\node[on chain=1, init]
  (f1) {$I_{1,1}$};
\node[on chain=1,init]
 (f3) {$I_{1,2}$};
\node[on chain=1,init] (f5) 
  {$I_{1,3}$};
\end{scope}

\begin{scope}[start chain=2,node distance=8mm]
\node[on chain=2, init2] at (0,-12mm)
(b1) {$\beta_1$};
\node[on chain=2, init2] 
  (b2) {$\beta_2$};
\node[on chain=2, init2] 
  (b3) {$\beta_3$};
 \node[on chain=2, init] 
  (g) {$\mathrm{MSE}$};
\end{scope}

\begin{scope}[start chain=1,node distance=8mm]
\node[on chain=1, init] at (0,-24mm)
  (f2) {$I_{2,1}$};
\node[on chain=1,init]
 (f4) {$I_{2,2}$};
\node[on chain=1,init] (f6) 
  {$I_{2,3}$};
\end{scope}

\draw[-latex] (f1) -- (f3);
\draw[-latex] (f1) -- (f4);
\draw[-latex] (f2) -- (f3);
\draw[-latex] (f2) -- (f4);
\draw[-latex] (f3) -- (f5);
\draw[-latex] (f3) -- (f6);
\draw[-latex] (f4) -- (f5);
\draw[-latex] (f4) -- (f6);

\draw[-latex] (b1) -- (f1);
\draw[-latex] (b1) -- (f2);
\draw[-latex] (b2) -- (f3);
\draw[-latex] (b2) -- (f4);
\draw[-latex] (b3) -- (f5);
\draw[-latex] (b3) -- (f6);

%\draw[-latex] (f1) -- (g);
%(f1) edge[bend right] node [left] {} (g);
\draw [-latex] (f1) to [out=40,in=100] (g);
\draw [-latex] (f3) to [out=45,in=115] (g);
\draw [-latex] (f5) to  (g);
\draw [-latex] (f2) to [out=320,in=260] (g);
\draw [-latex] (f4) to [out=315,in=245] (g);
\draw [-latex] (f6) to  (g);
%\draw[gray,step=1.0] (-1,-4) grid (6, 2);
%\clip (-1,-3) rectangle (6, 0);
%\draw [brown] (current bounding box.south west) rectangle (current bounding box.north east);
\end{tikzpicture} 
\quad
\begin{tikzpicture}[
init/.style={
  draw,
  circle,
  inner sep=0.7pt,
  minimum size=0.7cm
},
init2/.style={
  circle,
  inner sep=0.7pt,
  minimum size=0.7cm
},
]
\begin{scope}[start chain=1,node distance=8mm]
\node[on chain=1, init] 
  (f1) {$L_1$};
\node[on chain=1,init]
 (f4) {$L_2$};
\node[on chain=1,init] (f7) 
  {$L_3$};
\node[on chain=1,init] (s) 
{$\sum_t L_t$};
\end{scope}

\begin{scope}[start chain=1,node distance=8mm]
\node[on chain=1, init] at (0,-12mm)
  (f2) {$I_1$};
\node[on chain=1,init]
 (f5) {$I_2$};
\node[on chain=1,init] (f8) 
  {$I_3$};
\end{scope}

\begin{scope}[start chain=1,node distance=8mm]
\node[on chain=1, init] at (0,-24mm)
  (f3) {$R_1$};
\node[on chain=1,init]
 (f6) {$R_2$};
\node[on chain=1,init] (f9) 
  {$R_3$};
\end{scope}

\begin{scope}[start chain=2,node distance=8mm]
\node[on chain=2, init2] at (0,-36mm)
(x1) {$x_1$};
\node[on chain=2, init2] 
  (x2) {$x_2$};
\node[on chain=2, init2] 
  (x3) {$x_3$};
\end{scope}

\draw[-latex] (f2) -- (f4);
\draw[-latex] (f2) -- (f5);
\draw[-latex] (f2) -- (f6);
\draw[-latex] (f3) -- (f4);
\draw[-latex] (f3) -- (f5);
\draw[-latex] (f3) -- (f6);

\draw[-latex] (f5) -- (f7);
\draw[-latex] (f5) -- (f8);
\draw[-latex] (f5) -- (f9);
\draw[-latex] (f6) -- (f7);
\draw[-latex] (f6) -- (f8);
\draw[-latex] (f6) -- (f9);

%\draw[-latex] (x2) -- (f4);
%\draw[-latex] (x2) -- (f5);
%\draw[-latex] (x2) -- (f6);

\draw[-latex] (x1) to [out=135,in=205] (f1);
\draw[-latex] (x1) to [out=130,in=230] (f2);
\draw[-latex] (x1) -- (f3);

\draw[-latex] (x2) -- (f6);
\draw[-latex] (x2) to [out=130,in=230] (f5);
\draw[-latex] (x2) to [out=135,in=205] (f4);

\draw[-latex] (x3) to [out=135,in=205] (f7);
\draw[-latex] (x3) to [out=130,in=230] (f8);
\draw[-latex] (x3) -- (f9);

\draw[-latex] (f1) to [out=30,in=150] (s);
\draw[-latex] (f4) to [out=25,in=160] (s);
\draw[-latex] (f7) -- (s);
\end{tikzpicture}
\caption{Function network for the (left) epidemic calibration problem in \S\ref{sec:calibration} and (right) the COVID-19 pooled  testing optimization problem described in the supplement.
%There are three time periods. For each time period $t=1,2,3$, the fraction of the population that will be infected and recovered, respectively, at the end of the time period is computed by nodes $I_t$ and $R_t$. Additionally, a loss $L_t$ is computed that results from new infections, quarantined individuals (including false positives), and testing resources used.  These values are determined by the fraction of the population infected and recovered at the start of the period, provided as the output of the previous period (for $t=1$, these values are fixed and not shown), and the group testing pool size $x_t$.  Losses are summed across time periods to compute the overall objective.
  \label{fig:epidemic-covid}}
\end{figure}
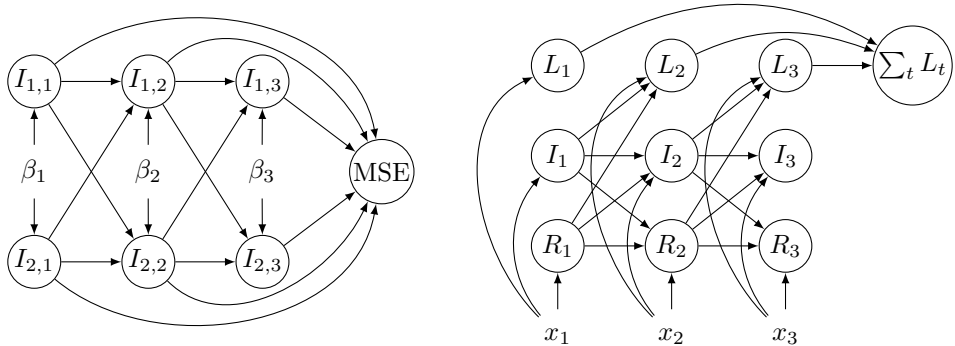

\subsection{Discussion}
\label{sec:results}
Across the wide range of problems considered, EI-FN significantly improves query efficiency over standard BO methods that ignore the function network structure of evaluations.
The benefits range from a ~5\% improvement 
in the value of the best point found on the Drop-Wave and manufacturing problems to several orders of magnitude in the Rosenbrock and epidemic model calibration problems.

The largest benefit arise when the control input is high-dimensional but the input to individual nodes is low-dimensional.
%This is consistent with the benefit being provided by leveraging additional information provided by output of internal nodes not available to standard BO methods.
On the epidemic model calibration problem, we see EI-CF \citep{astudillo2019bayesian} outperforming EI and KG by several orders of magnitude, and EI-FN outperforming EI-CF by several {\it additional} orders of magnitude. 
As noted above, EI-CF can be seen as a special case of EI-FN using a less informative function network that hides observations from some nodes. This is consistent with observations of function network structure allowing substantial improvement in query efficiency, and observing more of the internal function network structure providing more value.

%\section{Conclusion and Future Work}
\section{Conclusion}
\label{sec:conclusion}
We introduced a novel BO approach for objective functions defined by a series of expensive-to-evaluate functions, arranged in a network so that each function takes as input the output of its parent nodes. These objective functions arise in a wide range of application domains. However, existing methods cannot leverage this structure. Our approach models the outputs of the functions in this network instead of only the objective function, as is standard in BO. Our experiments show that, by doing so, this approach can dramatically  outperform standard BO methods.

Though we see substantial benefits from our approach, there are limitations. First, it requires more computation than standard BO methods, as explored in the supplement. (When the objective is time-consuming, the improved query efficiency more than makes up for the additional computation required.) Second, while we have demonstrated our method in problems with up to 9 nodes, and computational speed would support more, our method does not (yet) scale to hundreds of nodes. 
% We see generalizing in this direction as a substantial opportunity for future work.
Third, while we show consistency, it would be instructive to complement our empirical results showing fast convergence with a theoretical understanding of convergence rates. Existing approaches to prove convergence rates 
% We end this section by noting that existing convergence rates 
for the classical expected improvement heavily rely on properties of its analytical expression 
\citep{bull2011convergence,ryzhov2016convergence}, and thus are not directly generalizable to our setting. This is, however, an exciting direction for future work. 
% (cites) rely heavily on precise analytical forms of the acquisition function, making them a less useful foundation, but we nevertheless see this direction as a promising one for future research.

As with any powerful new method for optimizing time-consuming-to-compute black-box functions, ours can accelerate many applications. While this includes innovations that  generally benefit society, such as improvements to public health and vaccine manufacturing, it also includes the design of weapons and other engineering systems that could harm individuals. Thus, it is important that society enact guardrails that ensure proper use of our methodology.

\section*{Acknowledgments}
The authors were partially supported by AFOSR FA9550-19-1-0283 and FA9550-20-1-0351.

% An important direction of future work is to develop a mechanism to apply our approach in large function networks, where modeling all node functions would be too computationally expensive and perhaps not beneficial. We believe that, in such settings, it is often possible to select a subset of the nodes and still get substantial benefit. The ability to select this subset automatically would increase the applicability of our approach.

\bibliographystyle{apalike} 
\bibliography{bibl} 

\renewtheorem{theorem}{Theorem}[section]
\renewtheorem{lemma}{Lemma}[section]
\newtheorem{assumption}{Assumption}[section]
\newtheorem{definition}{Definition}[section]
\renewtheorem{proposition}{Proposition}[section]

\clearpage
\appendix
\section{Proof of Proposition 1}
\label{supp:prop1}
In this section, we provide a formal statement and proof of Proposition 1. We begin by proving the following auxiliary result.

\begin{lemma}
 Suppose that $f:\R^{B_1\times\ldots\times B_J}\rightarrow\R$ and $h_j:\R^{A}\rightarrow\R^{B_j}, \ j=1,\ldots, J,$ are all Lipischitz continuous with Lipschitz constants $L_f$ and $L_{h_j}, \ j=1,\ldots, J,$ respectively. Then, the function $g:\R^{A}\rightarrow\R$ defined by $g(x) = f( h_1(x),\ldots, h_J(x))$ is Lipschitz with Lipschitz constant $L_{g}:=L_{f}  \sum_{j=1,\ldots, J}L_{h_j}$.
\end{lemma}
\begin{proof}
We have 
\begin{align*}
    |g(x) - g(x')| &= |f(h_1(x),\ldots, h_J(x)) - f(h_1(x'),\ldots, h_J(x'))|\\
    & \leq L_{f}\|( h_1(x),\ldots, h_J(x)) - (h_1(x'),\ldots, h_J(x'))\|_2\\
    & \leq L_{f}  \sum_{j=1}^J\|h_j(x) - h_j(x')\|_2\\
    & \leq L_{f}  \sum_{j=1}^JL_{h_j}\|x - x'\|_2\\
    & = L_{g}\|x -  x'\|_2.
\end{align*}
\end{proof}

We are now in position to show Proposition 1, which can be seen as a simple generalization of Theorem 1 in \cite{balandat2020botorch}.

 \begin{proposition}[Proposition 1]
 \label{prop:saa}
 Suppose that $\domain$ is compact, and that the functions $\mu_{n,k}, \sigma_{n,k}:\R^{|I(k)|}\times \R^{|J(k)|} \rightarrow \R$, $k=1,\ldots,K$, are Lipschitz continuous. Let
 \begin{equation*}
     \widehat{x}_{*}^{(M)}\in\argmax_{x\in\domain} \widehat{\eifn}_n\left(x ; Z^{(1:M)}\right),
     \quad
     X_* =  \argmax_{x\in\domain}\eifn_n(x),
 \end{equation*}
 where $\{Z_m\}_{m=1}^\infty$ are independent standard normal random variables.
 Then, for every $\epsilon>0$, there exist $A, \alpha > 0$ such that $\prob\left(\textnormal{dist}\left(\widehat{x}_{*}^{(M)}, X_*\right)>\epsilon\right) \leq Ae^{-\alpha M}$ for all $M$.
 \end{proposition} 
 
 \begin{proof}
 Let $L_{\mu_{n,k}}$ and $L_{\sigma_{n,k}}$ be the Lipschitz constants of $\mu_{n,k}$ and $\sigma_{n,k}$, respectively, and consider the functions $\widehat{f}_{n,k}:\R^{|I(k)|}\times \R^{|J(k)|} \times\R \rightarrow \R, \ k=1,\ldots, K,$ given by
 \begin{equation*}
     \widehat{f}_{n,k}(x_{I(k)}, y_{J(k)}, z_k) = \mu_{n,k}(x_{I(k)}, y_{J(k)})+ \sigma_{n,k}(x_{I(k)}, y_{J(k)})z_k.
 \end{equation*}
 
 We note that, for any fixed $z_k$, the function $(x_{I(k)}, y_{J(k)})\mapsto\widehat{f}_{n,k}(x_{I(k)}, y_{J(k)}, z_k)$ is $\left(L_{\mu_{n,k}} + L_{\sigma_{n,k}}|z_k|\right)$-Lipschitz.
 
 Now consider the functions $\widehat{h}_{n,1},\ldots, \widehat{h}_{n,k}:\domain\times\R^K \rightarrow \R$ defined recursively by
 \begin{equation*}
     \widehat{h}_{n,k}(x, z) = \widehat{f}_{n,k}\left(x_{I(k)}, \widehat{h}_{n,J(k)}(x,z), z_k\right), \ k=1,\ldots, K.
 \end{equation*}
 Applying Lemma 1 repeatedly, we find that, for any fixed $z\in\R^K$, the functions $x\mapsto\widehat{h}_{n,k}(x,z), \ k=1,\ldots, K,$ are Lipschitz continuous with Lipschitz constants $L_{\widehat{h}_{n,k}}(z), \ k=1,\ldots, K,$ defined recursively by
 \begin{equation*}
    L_{\widehat{h}_{n,k}}(z) =  \left(L_{\mu_{n,k}} + L_{\sigma_{n,k}}|z_k|\right)\left(1 + \sum_{j\in J(k)}L_{\widehat{h}_{n,j}}(z)\right), \ k=1,\ldots, K.
 \end{equation*}
 
Let $\widehat{g}_n = \widehat{h}_{n,k}$ and $L_{\widehat{g}_n} = L_{\widehat{h}_{n,k}}$. Then, for any fixed $z$, the function $x \mapsto \widehat{g}_n(x, z)$ is $L_{\widehat{g}_n}(z)$-Lipschitz. Moreover, by definition, $\widehat{g}_n$ satisfies
\begin{equation*}
    \eifn_n(x) = \E_n\left[\left\{\widehat{g}_n(x, Z) - g_n^*\right\}^+\right],
\end{equation*}
 where $Z$ is a $K$-dimensional standard normal random vector and also
 \begin{equation*}
    \widehat{\eifn}_n\left(x ; Z^{(1:M)}\right) = \frac{1}{M}\sum_{m=1}^{M}\left\{\widehat{g}_n\left(x,Z^{(m)}\right) - g_n^*\right\}^+.
 \end{equation*}
 
 Now observe that $L_{\widehat{g}_n}(z)$ is a polynomial in the variables $|z_1|,\ldots, |z_k|$ with degree at most 1 for each variable. Since the folded normal distribution has finite moment generating function everywhere Therefore, if $Z$ is $K$-dimensional standard normal random vector, then $L_{\widehat{g}_n}(Z)$ has a finite moment generating function in a neighborhood of 0. 
 
 An similar argument can  be used to show that, for every $x$, $\widehat{g}_n(x, Z)$ has a finite moment generating function in a neighborhood of zero.  The desired result is now a direct consequence of Proposition 2 in the supplement of \citet{balandat2020botorch}, which is in turn a consequence of Theorem 2.3 in \citet{homem2008}.
 \end{proof}

 \section{Proof of Theorem 1}
 \label{supp:thm1}
 % (Almost Sure Consistency Under the Prior)
 \label{sec:theorem1}
 In this section, we prove Theorem 1. Throughout this section, we let $(x_n : n)$ denote the sequence of points at which the function network is evaluated. We begin by introducing the following definition, which is analogous to Definition 2.1 in \cite{bect2019supermartingale}.
 
 \begin{definition}
 Let $\mathcal{F}_n$ be the sigma algebra generated by the function network observations up to time $n$. The sequence $(x_n : n)$ is said to be a (non-randomized) sequential design if $x_n$ is $\mathcal{F}_{n-1}$-measurable for all $n$.
 \end{definition}
 
 Throughout this section, we assume that $(x_n : n)$ is a sequential design. We note, in particular, that if $x_n \in \max_{x\in\X}\eifn_{n-1}(x)$ for all $n$, then $(x_n : n)$ is a sequential design. 
 
 Our proof relies on the following assumptions.

\begin{assumption}
\label{a:compact}
$\mathbb{X}$ is compact.
\end{assumption}

\begin{assumption}
\label{a:continuous}
The prior mean and covariance functions of $f_1,\ldots, f_K$, are such that $f_1,\ldots, f_K$ are continuous almost surely.
\end{assumption}

\begin{assumption}
\label{a:bounded}
The prior covariance functions of $f_1,\ldots, f_K$ are bounded.
\end{assumption}

%\begin{assumption}
%With probability 1 under the prior, $\sup_{n\in\N, x\in\X,y\in\R^K}\mu_{n,K}(x_{I(K)},y_{J(K)})$ is finite.
%\end{assumption}

 \begin{assumption} \label{a:DCT}
 With probability 1 under the prior, given $f_1,\ldots, f_K$ and a sequential design $(x_n:n)$, there exists a function $\beta(z)$ such that $|\widehat{g}_{n-1}(x_n,z)| \le \beta(z)$ for all $n$ and $z$ and $\int_{-\infty}^\infty \varphi(z) \beta(z) < \infty$, where $\varphi$ is the standard normal pdf.
 \end{assumption}
 
Assumptions~\ref{a:compact}, ~\ref{a:continuous}, and ~\ref{a:bounded} are standard.
Assumption~\ref{a:DCT} is bespoke to our arguments, but holds, for example, when the posterior mean of $f_K$ is uniformly bounded.  (This bound can depend on $f_1,\ldots, f_K$.)
% e.g., by the sum of $\sup_{x\in \mathbb{X}} g(x)$ and some maximum deviation between the posterior mean and $g(x)$. 
This occurs, for example, when each $f_{k}$ is 
% almost surely
in the reproducing kernel Hilbert space (RKHS) corresponding to the prior covariance function. In particular, if $f_K$ is in this RKHS and the prior kernel is bounded, then $f_K$ is bounded
and there is a uniform bound (depending on the RKHS norm of $f_K$ and the prior covariance) over the deviation between $f_K$ and the sequence of posterior means resulting from our observations. 
The sum of these two bounds and a term that is linear in $z$ arising from the posterior standard deviation term in the definition of $\widehat{g}_{n-1}(x_n,z)$ provides $\beta(z)$. % (This bound depends on $z$ through the dependence of $f_K$ This particular $\beta(z)$ does not depend on $z$.
% In this case, we can take $\beta(z) = g(z) + C$
% C = sup of deviation over 
We also believe that Assumption~\ref{a:DCT} holds more broadly.

% sequence of posterior mean functions.

%\begin{lemma}
%Suppose that $\f_k\in\mathcal{H}_k, \ k=1,\ldots, K$. Then, there exist real numbers $a_k$ and $b_k$ (not depending on $\{x_n\}_{n=1}^\infty$) such that
%\begin{equation*}
    %|\mu_{n,k}(x_{I(k)}, \widehat{h}_{n,J(k)}(x,z))|\leq a_k, \textnormal{ and } \sigma_{n,k}(x_{I(k)}, \widehat{h}_{n,J(k)}(x,z)) \leq b_k, \ k=1,\ldots, K,
%\end{equation*}
%for all $x\in\X$, $z\in\mathbb{Z}$, and $n\in\N$.
 %\end{lemma}
 %\begin{proof}
 %Note that 
 %\begin{equation*}
    %\sigma_{n,k}(x_{I(k)}, \widehat{h}_{n,J(k)}(x,z)) \leq  \sigma_{0,k}(x_{I(k)}, \widehat{h}_{n,J(k)}(x,z)) \leq \sup_{(x_{I(k)},y_{J(k)})}\sigma_{0,k}(x_{I(k)},y_{J(k)}) < \infty.
 %\end{equation*}
 %Therefore, it is enough to take $b_k = \sup_{(x_{I(k)},y_{J(k)})}\sigma_{0,k}(x_{I(k)},y_{J(k)})$.
 
%On the other hand, the Cauchy-Schwarz inequality in $\mathcal{H}_k$ implies that
 %\begin{equation*}
     %|\mu_n,k(\mu_{n,k}(x_{I(k)}, \widehat{h}_{n,J(k)}(x,z))) - f_k(x_{I(k)}, \widehat{h}_{n,J(k)}(x,z))) | \leq \sigma
 %\end{equation*}
 %\end{proof}
  
Note that our proof does not rely on the no-empty-ball assumption (NEB) of \citet{vazquez2010convergence}. Thus, our proof also extends the proof of \citet{astudillo2019bayesian} to a broader class of prior distributions.
  
  As in the main paper, we refer to the ``time-$n$'' posterior, which is the conditional distribution of $f_1,\ldots, f_K$ given $(x_{m}:m\leq n)$ and $(h_k(x_m): k\le K, m\le n)$.
  $\E_n$ denotes expectation with respect to this conditional distribution, and $\prob_n$ denotes the probability operator.

Recall the sampling procedure from Section 4.3 of the main paper. This defined a function $\widehat{g}(x,Z)$ that depended on the current posterior distribution, a point $x$ in the feasible domain, and a vector $Z$. When $Z$ was generated as a standard normal random variable, the distribution of $\widehat{g}(x,Z)$ was the same as that of $g(x)$ under the current posterior.
To support working with this over a sequence of posterior distributions, we use $\widehat{g}_n(x,Z)$ to indicate this function calculated using the posterior at time $n$.
We similarly use the notation $\widehat{h}_{k}(x,Z)$ to represent the function $\widehat{h}_{n,k}(x,Z)$ from the main paper computed with respect to the posterior at time $n$. 
  
  % In the proof of Theorem~1 and the lemmas that lead up to it, we consider a particular subsequence of the sequence of points queried $(x_n: n) \subset \mathbb{X}$. To construct this sequence,  let $r_m$ be a sequence of real numbers decreasing to 0. Then, let $(n_m : m)$ be a subsequence of $(x_n : n)$ such that the the distance between $x_{n_m}$ and the closest previously measured point is less than $r_m$. This is possible because $(x_n : n)$ is contained in a compact set $\mathbb{X}$ and so has a convergent subsequence. Let $x_{n_\infty} = \lim_{m\to\infty} x_{n_m}$ be the limit of this subsequence.
  
\begin{lemma}\label{lem:g-infinity}
$g^*_\infty := \lim_n g^*_n$ exists and is finite almost surely. Moreover, any limit point $x_\infty$ of the sequence $(x_n : n)$ satisfies $g(x_\infty) \le g^*_\infty$.
\end{lemma}
\begin{proof}
The sequence $(g^*_n : n)$ is non-decreasing and bounded above by the random variable $g^* := \max_{x' \in \mathbb{X}} g(x')$. The random variable $g^*$ is almost surely finite since $g$ is almost surely continuous (it is the composition of a collection of almost surely continuous functions $h_k$)
and $\mathbb{X}$ is compact.
Thus $g^*_\infty := \lim_n g^*_n$ exists and is finite almost surely.

Let $x_\infty$ be the limit of a convergent subsequence 
$(x_{n_m} : m)$ of $(x_n : n)$. 
Since $g$ is almost surely continuous,
\begin{equation*}
g(x_\infty) = g(\lim_m x_{n_m}) = \lim_m g(x_{n_m}) \le \lim_m g^*_{n_m+1} \le \lim_m g^*_\infty = g^*_\infty.  
\end{equation*}
\end{proof}

\begin{lemma}
Consider the almost sure event that $f_{k}$ is continuous for all $k = 1,\ldots,K$. On this event, the function $h_k$ is continuous 
% and the set $\{h_k(x) : x \in \mathbb{X} \}$ is compact 
for all $k = 1,\ldots,K$.
\label{lem:h-continuous}
\end{lemma}
\begin{proof}
% Since $\mathbb{X}$ is compact and the image of a compact set through a continuous function is compact, it is sufficient to show that $h_k$ is continuous for all $k$ on the event considered.

We show this via induction.
The base case, for $k=1$, follows since $f_1$ is continuous on the event considered, $x_{I(k)}$ is a continuous function of $x$, and so $h_1(x) = f_1(x_{I(k)})$ is the composition of two continuous functions and so is a continuous function of $x$.

We now show the induction step.  Fix $k>1$. Suppose $h_{k'}(x)$ is continuous for all $k'<k$.
applying the induction hypothesis for all $k'\in J(k) \subseteq \{1,\ldots,k-1\}$ implies that $h_{J(k)}(x)$ is continuous. Also $x_{I(k)}$ is a continuous function of $x$.
Thus, $x\mapsto (x_{I(k)}, h_{J(k)}(x))$ is continuous.
This and the fact that $f_k$ is continuous on the event considered implies that 
$h_{k}(x) = f_k(x_{I(k)}, h_{J(k)}(x))$ is a composition of continuous functions and so is continuous.
\end{proof}

\begin{lemma}
\label{lem:unifconv}
For each $k=1,\ldots, K$, the functions $\mu_{n,k}$ and $\sigma_{n,k}$ converge pointwise to some continuous functions $\mu_{\infty, k}$ and $\sigma_{\infty, k}$ almost surely; moreover, this convergence is uniform over compact subsets of $\R^{|I(k)|}\times \R^{|J(k)|}$.
\end{lemma}

\begin{proof}
Fix $k$ and consider the almost sure event that $f_{1},\ldots,f_{k-1}$ are continuous. Condition on continuous realizations of $f_{1},\ldots,f_{k-1}$, thus fixing $h_{J(k)}$ as well.

Let $A$ be an arbitrary compact subset of $\R^{|I(k)|}\times \R^{|J(k)|}$ and define $B = \{ (x_{I(k)}, h_{J(k)}(x)) : x \in \mathbb{X} \}$.  $B$ is compact by Lemma~\ref{lem:h-continuous} since $x\mapsto (x_{I(k)}, h_{J(k)}(x))$ is continuous, $\mathbb{X}$ is compact, and the image of a compact set through a continuous function is compact.
The observations of $f_k$ occur at input points $\{(x_{n,I(k)}, h_{J(k)}(x_n)) : n\} \subset B$.

Let $C = A \cup B$ and note that $C$ is compact. Then, by Proposition 2.9 in \cite{bect2019supermartingale}, $\mu_{n,k}$ and $\sigma_{n,k}$ converge uniformly over $C$ and thus also over $A$.
\end{proof}

 % \begin{lemma}
 % \label{lem:obvious}
 % Suppose we have a GP and a sequence of measurements $x_n$ in a compact space. Suppose we have a convergent subsequence $x_{n_m}$ with a limit point $x_\infty$.  Then,
 % $\lim_m \mu_{n_m-1}(x_n) = f(x_\infty)$
 % and $\lim_m \sigma_{n_m-1}(x_n) = 0$.
 % \end{lemma}
 % \begin{proof}
 % To simplify notation, let $x(m) = x_{n(m)}$ and $n(m) = n_m - 1$.
 % Then,
% $E_{n(m)[(f(x_\infty) - \mu_{n_m-1}(x_{n_m}))^2]
% \le  
% $E_{n_m-1}[(f(x_\infty) - f(x_{n_m}))^2] + 
% E_{n_m-1}[(f(x_{n_m}) - \mu_{n_m-1}(x_{n_m}))^2]
% \mu_{n_m-1}(x_n))^2]
% $
% NOT DONE --- IDEA IS TO FOLLOW VAZQUEZ AND BECHT FOR THE 
 % \end{proof}

\begin{lemma} \label{lem:lim_fixed_z}
Let $(x_{n_m} : m)$ be a convergent subsequence of $(x_n : n)$ with limit $x_\infty$. Then, $\lim_{m\to\infty} \widehat{g}_{n_m-1}(x_{n_m},z) = g(x_\infty)$ for each $z\in\R^K$ almost surely.
 \end{lemma}
\begin{proof}
All the convergence claims made in this proof are almost surely.
First note that Lemma~\ref{lem:unifconv} implies that, for each $k=1,\ldots,K$, the function $\widehat{f}_{n,k}$ defined in the proof of Proposition 1 converges to the function $\widehat{f}_{\infty, k}$ defined by 
 \begin{equation*}
  \widehat{f}_{\infty, k}(x_{I(k)}, y_{J(k)}, z_k) = \mu_{\infty,k}(x_{I(k)}, y_{J(k)})+ \sigma_{\infty,k}(x_{I(k)}, y_{J(k)})z_k.   
 \end{equation*}
 uniformly in $(x_{I(k)}, y_{J(k)})$ (but not necessarily $z_k$). This in turn implies that, for each $k=1,\ldots,K$, $\widehat{h}_{n,k}$ converges to the function $\widehat{h}_{\infty, k}$ defined recursively by
 \begin{equation*}
     \widehat{h}_{\infty,k}(x, z) = \widehat{f}_{\infty, k}\left(x_{I(k)}, \widehat{h}_{\infty,J(k)}(x,z), z_k\right)
 \end{equation*}
 uniformly in $x \in \X$ for each $z \in \mathbb{R}^K$.

 We show the following two claims by induction on $k$:
 \begin{enumerate}
 \item $\sigma_{n_m-1,k}(x_{n_m,  I(k)}, \widehat{h}_{n_m-1,J(k)}(x_{n_m},z))  \to 0$.
 \item  $\widehat{h}_{n_m-1,k}(x_{n_m},z) \to h(x_\infty)$.
 \end{enumerate}
 
 We first show the induction step, where we assume the induction hypothesis is true for all $k'<k$ and show it for $k$.
 
 Each element of $J(k)$ is strictly less than $k$ and so the induction hypothesis implies that $\widehat{h}_{n_m-1,J(k)}(x_{n_m},z) \rightarrow h_{J(k)}(x_\infty)$. Moreover, since $x_{n_m} \rightarrow x_\infty$, and $\sigma_{n,k}$ converges uniformly to $\sigma_{\infty, k}$, it follows that $\sigma_{n_m-1,k}(x_{n_m,  I(k)}, \widehat{h}_{n_m-1,J(k)}(x_{n_m},z))$ converges to $\sigma_{\infty, k}(x_{\infty,I(k)}, h_{J(k)}(x_\infty))$. Similarly, $\widehat{h}_{n_m-1,k}(x_{n_m},z)$ converges to 
\begin{align*}
\widehat{h}_{\infty,k}(x_\infty, z) &= \widehat{f}_{\infty, k}\left(x_{\infty,I(k)}, \widehat{h}_{\infty,J(k)}(x_\infty,z), z_k\right)\\
&= \widehat{f}_{\infty, k}\left(x_{\infty,I(k)}, h_{J(k)}(x_\infty), z_k\right)\\
&= \mu_{\infty,k}(x_{\infty,I(k)}, h_{J(k)}(x_\infty))+ \sigma_{\infty,k}(x_{\infty,I(k)}, h_{J(k)}(x_\infty))z_k,
\end{align*} 
where the second equation is obtained by noting that, since $\widehat{h}_{n_m-1,J(k)}(x_{n_m},z) \rightarrow h_{J(k)}(x_\infty)$ (by the induction hypothesis) and $\widehat{h}_{n_m-1,J(k)}(x_{n_m},z) \rightarrow \widehat{h}_{\infty,J(k)}(x,z)$, it must be the case that $\widehat{h}_{\infty,J(k)}(x_\infty,z) = h_{J(k)}(x_\infty)$.

Now observe that  $\sigma_{n_m+1,k}(x_{n_m,  I(k)}, h_{J(k)}(x_{n_m}))$ converges to $\sigma_{\infty,k}(x_{\infty,I(k)}, h_{J(k)}(x_\infty))$, but $\sigma_{n_m+1,k}(x_{n_m,  I(k)}, h_{J(k)}(x_{n_m}))=0$ for all $m$, and thus $\sigma_{\infty,k}(x_{\infty,I(k)}, h_{J(k)}(x_\infty)) = 0$. This proves the first part of the induction step. Similarly, note that $\mu_{n_m+1,k}(x_{n_m,  I(k)}, h_{J(k)}(x_{n_m}))$ converges to $\mu_{\infty,k}(x_{\infty,I(k)}, h_{J(k)}(x_\infty))$, but 
\begin{align*}
\mu_{n_m+1,k}(x_{n_m,  I(k)}, h_{J(k)}(x_{n_m})) &= f_k(x_{n_m,  I(k)}, h_{J(k)}(x_{n_m}))\\
&= h_k(x_{n_m}) \rightarrow h_k(x_\infty).
\end{align*}
This proves the second part of the induction step.

The proof of the base case ($k=0$) is analogous except that $J(k)=\emptyset$ eliminates terms that depend on $k'<k$.

\end{proof}

 \begin{lemma}
 \label{lem:lim-inf-EI}
 $\liminf_{n} \eifn_{n-1}(x_{n}) = 0$ almost surely.
 \end{lemma}
 \begin{proof}
 Since $(x_n : n)$ is contained in a compact set, it has a convergent subsequence, $(x_{n_m}:m)$. 
 
 Then, letting $\varphi(z)$ be the standard normal probability density function,
 \begin{align*}
 \lim_{m\to\infty} \eifn_{n_m-1}(x_{n_m})
 &= \lim_{m\to\infty} \int_{-\infty}^\infty 
 \left\{ \widehat{g}_{n_m-1}(x_{n_m},z) - g^*_{n_m-1} \right\}^+
 \varphi(z)\, dz \\
&= \int_{-\infty}^\infty \lim_{m\to\infty} 
 \left\{ \widehat{g}_{n_m-1}(x_{n_m},z) - g^*_{n_m-1} \right\}^+
 \varphi(z)\, dz 
 \end{align*}
 by the dominated convergence theorem and Assumption~\ref{a:DCT}.
 
 By Lemmas~\ref{lem:g-infinity} and~\ref{lem:lim_fixed_z},
 \begin{equation*}
\lim_{m\to\infty} 
\left\{ \widehat{g}_{n_m-1}(x_{n_m},z) - g^*_{n_m-1} \right\}^+ 
= \left\{ g(x_\infty) - g^*_\infty \right\}^+ \\
=0 
 \end{equation*}
 for each $z$.  
 Thus,
 $\lim_{m\to\infty} \eifn_{n_m-1}(x_{n_m}) = 0$, implying that $\liminf_n \eifn_{n-1}(x_n) \le 0$.
 This and the fact that $\eifn_{n-1}(x) \ge 0$ for all $x$ implies that 
$\liminf_n \eifn_{n-1}(x_n) = 0$.
 \end{proof}

 The following lemma considers a sequence of random variables $I_n$ that we will later take to be the random improvements generated within our statistical model under the posterior after $n$ measurements. The quantity $\E[I_n^+]$ will then be the EI-FN under this posterior.
 \begin{lemma}\label{lem:improvement-cdf}
 % Consider a sequence of probability distribution $Q_n$, each over the real line. Let $I_n$ be a random variable distributed according to $Q_n$. 
 Consider a sequence of scalar random variables $I_n$.
 If $\liminf_{n} \prob(I_n \ge \epsilon) > 0$ for any given $\epsilon>0$, then $\liminf_{n} \E[I_n^+] > 0$.
 \end{lemma}
 \begin{proof}
We have $\E[I_n^+] \ge \epsilon \prob(I_n \ge \epsilon)$.
Thus, $\liminf_{n} \E[I_n^+] \geq \epsilon\liminf_n  \prob(I_n \ge \epsilon) > 0$.
 \end{proof}
 
 We are now ready to prove Theorem 1.
 \begin{proof}[Proof of Theorem 1]
 Pick any point $x\in\mathbb{X}$.  Since we choose to evaluate at the point with largest $\eifn_n(x)$, $\eifn_n(x) \le \eifn_n(x_{n+1})$ for each $n$.
 
 Lemma~\ref{lem:lim-inf-EI} then implies that there is a subsequence $(n_m : m)$ on which $\lim_m \eifn_{n_m}(x_{n_m+1}) = 0$. This and the fact that $\eifn_n(x) \ge 0$ imply that $\lim_n \eifn_{n_m}(x) = 0$. Thus, $\liminf_n \eifn_n(x) = 0$.
 
Recall that $\eifn_n(x) = \E_n[\left\{g(x) - g^*_n\right\}^+]$.
 Consider the conditional distribution of $g(x) - g^*_n$ under the time-$n$ posterior.
 This is the same as the conditional distribution of 
$\widehat{g}_n(x;Z) - g^*_n$ where only $Z$ is random and the other quantities are completely determined by the observations of the function network at $x_1,\ldots, x_n$.
By taking $I_n$ to be a random variable with the same distribution for each $n$,
then on any sequence of observations,
Lemma~\ref{lem:lim-inf-EI} and the contrapositive of Lemma~\ref{lem:improvement-cdf} imply that
 $\liminf_n \prob_n(g(x) - g^*_n \ge \epsilon) = 0$ for each $\epsilon>0$.
 
 Since the random variable $g^*_\infty$ defined in Lemma~\ref{lem:g-infinity} bounds each $g^*_n$ above, 
$\prob_n(g(x) - g^*_n \ge \epsilon) \ge \prob_n(g(x) - g^*_\infty \ge \epsilon)$ and we have 
 $\liminf_n \prob_n(g(x) - g^*_\infty \ge \epsilon) = 0$ for each $\epsilon>0$.
 
For any event $W$, $(\prob_n(W) : n)$ is a uniformly integrable martingale, and thus converges almost surely to a limiting random variable $\prob_\infty(W)$, where $\prob_\infty$ is defined as the conditional expectation with respect to the event $(x_n, h_k(x_n) : n<\infty, k\leq K)$ (by Theorem 5.13 of \cite{cinlar}).
Taking $W$ to be the event that $g(x) - g^*_\infty \ge \epsilon$, we have 
that $\prob_n(g(x) - g^*_\infty \ge \epsilon)$ has a limit,
$\prob_\infty(g(x) - g^*_\infty \ge \epsilon)$.
Moreover, this limit must be the same as the $\liminf$, which we showed above was 0.
Thus,
\begin{equation*}
\prob_\infty(g(x) - g^*_\infty \ge \epsilon) = 0.
\end{equation*}

Since this is true for each $\epsilon>0$, taking the limit as $\epsilon\to0$ and using the monotone convergence theorem shows 
\begin{equation*}
\prob_\infty(g(x) > g^*_\infty) = 0.
\end{equation*}

% $\liminf_n P_n(g(x) \le g^*_\infty) = 1$.
 
%  This implies that 
%  $\liminf_n P_n(g(x) \le g^*_n) = 1$.

% The random variable $g^*_\infty$ defined in Lemma~\ref{lem:g-infinity} bounds each $g^*_n$ above, so $\liminf_n P_n(g(x) \le g^*_\infty) = 1$.

% For any event $W$, $(P_n(W) : n)$ is a uniformly integrable martingale, and thus converges almost surely to a limiting random variable $P_\infty(W)$, where $P_\infty$ is defined as the conditional expectation with respect to the infinite set $(x_n, h_k(x_n) : n, k)$ (by Theorem 5.13 of \citealt{cinlar}). Taking $W$ to be the event that $g(x) \le g^*_\infty$, we have that $P_\infty(g(x) \le g^*_\infty) = 1$.

Taking the expectation under the prior and applying the law of conditional expectation, we have that
\begin{equation*}
  0 
= \E\left[\prob_\infty(g(x) > g^*_\infty)\right]
= \E\left[\E_\infty(1\{g(x) > g^*_\infty\})\right] 
= \E\left[1\{g(x) > g^*_\infty\}\right] 
= \prob(g(x) > g^*_\infty).  
\end{equation*}
Thus, the value of $g(x)$ is almost surely less than or equal to the limiting value of the sequence of best points found.

Let $X$ be a countable set that is dense in $\mathbb{X}$. Such set exists because $\mathbb{X}$ is compact. Then, because the countable union of events with probability zero also has probability zero,
\begin{equation*}
    0 = \prob(g(x) > g^*_\infty \textnormal{ for some } x\in X) =   \prob\left(\sup_{x \in X} g(x) > g^*_\infty\right)
\end{equation*}

Moreover, because $g$ is almost surely continuous
and $X$ is dense in $\mathbb{X}$,
$\sup_{x\in \mathbb{X}} g(x) = \sup_{x \in X} g(x)$ almost surely. Hence, $\prob\left(\sup_{x \in \mathbb{X}} g(x) > g^*_\infty\right) = 0$, which concludes the proof.
 \end{proof}

 \section{Proof of Proposition 2}
 \label{supp:prop2}
 In this section we prove Proposition 2 by providing a function network and a set of initial conditions where EI-FN does not measure the optimization domain densely. While the example we provide is very simple, we think such behavior also arises in more complex networks.
 
\begin{proposition}[Proposition 2]
There exists a function network in which EI-FN is consistent but whose measurements are not necessarily dense in $\domain$.
\end{proposition}
\begin{proof}
Let $\X=[0,1]$ and consider a function network with two nodes $f_1:\X\to\R$ and $f_2:\X\times\R\to\R$ where $f_2$ is deterministic, given by$ f_2(x,y)=\max\{1,y\}-x$, and the objective function is given by $g(x)=f_2(x, f_1(x))$. 

Suppose that $f_1$ is drawn from a GP prior with a continuous mean function and a bounded positive definite covariance function whose sample paths are almost surely continuous. From this and the fact that $f_2$ is deterministic and continuous, it follows that Assumptions 2.1-2.3 in \S2 are satisfied. Assumption 2.4 is also satisfied because $f_2$ is bounded over $\X\times\R$. Thus, Theorem 1 implies that EI-FN is consistent almost surely.

%We will consider EI-FN where it seeks to maximize $g(x)$ leveraging knowledge of $f_2$ and using a correctly specified prior for $f_1$.

Let $\tau = \inf\{ n\ge 1 : f_1(x_n) > 1\}$ be the first time that we measure a point whose value for $f_1$ is strictly greater than 1.  (If we never measure such a point, then $\tau$ is infinity.)
% Consider the event that the set of evaluations so far includes a point $x_{\triangle}\in[0,1)$ such that $f_1(1) > 1$.

If $\prob(\tau < \infty) = 0$ under EI-FN, then this problem is one in which EI-FN does not measure densely. This is because there is a strictly positive probability that $\sup_{x\in[0,1]} f_1(x) > 1$. Moreover, the fact that $f_1$ is almost surely continuous implies that on this event there is an  non-empty interval on which $f_1(x)$ is strictly above 1 over the entire interval. If EI-FN were to never measure in this interval then it would not have measured densely.
Thus, going forward, we assume $\prob(\tau < \infty) > 0$. In fact, using a similar argument, we may assume $\prob(\tau < \infty, \ f_1(0) < 1, \ f_1(1) < 1) > 0$.

% \newcommand{\lb}{\underline{x}}
% Let $\lb_n = \inf\{ x_m : m\le n, f_1(x_m) \ge 1 \}$. Below, $g(\lb_n)$ will be a lower bound on the global optimum.

For any $x > x_\tau$, we have $g(x) \le 1-x < 1-x_\tau = g(x_\tau)$ almost surely. Thus, any such $x$ is almost surely strictly suboptimal under the posterior and $\eifn_{n}(x) = 0$ for all $n \geq \tau$.

Now let $n \geq \tau$ and consider any unmeasured point $x$ with $1-x > g^*_n$; such a point exists on the event under consideration since $f_1(0) < 1$ implies $g^*_n \le \max\{f_1(0), 1 - \min_{m\le n} x_m\} < 1$ and we make take $x$ arbitrarily close to $0$.  Because the prior covariance function is positive definite, the posterior probability distribution over $f_1(x)$ has full support over the real line. Thus, in particular, there is a strictly positive posterior probability that $f_1(x) \ge 1$. On this event, $g(x) = 1-x > g^*_n$ and so $\eifn_n(x)$ is strictly positive. 

It follows that EI-FN would not measure at a point in the interval $(x_\tau,1]$, which concludes the proof. 

%This argument holds for any $n\ge \tau$ and so we have shown on an event with strictly positive probability (the event that $\tau<\infty$, $f_1(0) < 1$, and $f_1(1)<1$), EI-FN does not measure densely in $\mathbb{X} = [0,1]$. 

\end{proof}

\section{Additional Details on the Numerical Experiments}
\label{supp:details_numerical}
\subsection{Hyperparameter Estimation, Number of q-MC Samples, Runtimes, and Licenses}
All GPs in our experiments  have a constant mean function and ARD Mat\'ern covariance function with smoothness parameter equal to 5/2, which is a standard choice in practice. The length scales of these GPs are estimated via maximum a posteriori (MAP) estimation with Gamma priors.

We use $M=128$ quasi-MC samples obtained via scrambled Sobol sequences (see \citet{balandat2020botorch} for details) for computing the SAA of EI-FN. We use the same number of samples for EI-CF and 8 samples for KG. KG is maximized following the one-shot approach proposed introduced by \citet{balandat2020botorch}. Under this approach, the dimension of the optimization problem that arises when optimizing KG grows linearly with the number of samples and thus one is restricted to a small number of samples. The average runtimes of the BO methods for each of the problems are summarized in Table~\ref{table:runtimes}. We emphasize that, while optimizing EI-FN is more expensive than optimizing EI, the additional computation required by our method is compensated by its excellent performance, and is thus justified for problems where each function network evaluation takes several minutes or more. We also note KG becomes very expensive to optimize for problems with relatively high input dimension, and can be even more expensive than EI-FN.

The BoTorch python package and the source code for the robot pushing problem are both publicly available under a MIT licence. Our code is also publicly available under a MIT license.

\begin{table*}[th]
\caption{Average runtimes (seconds) per evaluation of the BO methods compared. EI-CF is N/A in problems that lack the structure it requires for use: that the objective is a composition of an inner black-box function and a  outer known non-linear function.}
\label{table:runtimes}
\begin{center}
\begin{tabular}{lrrrrrrr}
\toprule

  & KG & EI &  EI-CF &  EI-FN\\
\midrule

Drop-Wave  &  $15.1$ & $2.5$ & N/A & $15.4$\\
Rosenbrock, $K=4$    & $23.6$  & $4.16$ & N/A & $122.2$\\
Ackley    & $72.1$ & $18.3$ & N/A & $89.2$\\
Alpine2, $K=6$        & $84.1$ & $22.5$ &  N/A & $215.6$ \\
Manufacturing & $29.1$ & $5.4$ &  N/A & $117.2$ \\
COVID-19 & $43.1$ & $6.2$ &  N/A & $229.4$\\
Robot & $935.2$ & $29.6$ &  $110.1$ & $182.3$\\
Calibration & $1225.2$ & $43.5$ &  $207.2$ & $293.7$\\
\bottomrule
\end{tabular}
\end{center}
\end{table*}
% Expelimental setup details to be moved to supplement:
% How we're estimating hyperparameters
% Total amount of compute, and type of resources
% Licenses of any assets
% All GPs in our experiments  have a constant mean function and ARD Mat\'ern covariance function with smoothness parameter equal to 5/2, which is a standard choice in practice.
% We use $M=128$ samples when maximizing EI-FN and EI-CF via the SAA approach described above. 
 \subsection{Details on Synthetic Test Problems}
 Here we describe in detail how each of the synthetic test functions is arranged as a function network.
 
 \subsubsection{Alpine2}
 The Alpine2 test function \citep{jamil2013literature} is defined by
\begin{equation*}
    g(x) = -\prod_{k=1}^K\sqrt{x_k}\sin(x_k).
\end{equation*}
We adapt this function to our setting by letting 
\begin{equation*}
    f_1(x_1) =  -\sqrt{x_1}\sin(x_1);
\end{equation*}
\begin{equation*}
  f_k(x_k, y_{k-1}) = \sqrt{x_k}\sin(x_k)y_{k-1}, \ k = 2,\ldots, K;  
\end{equation*}
 $I(k)=\{k\}, \ k=1,\ldots, K$; $J(1)=\emptyset$; and $J(k) = \{k-1\}, \ k = 2,\ldots, K$. In our experiments, we set  $\domain=[0,10]^K$,  and consider $K=2, 4, \textnormal{ and } 6$.

The network structure of this test function can be summarized as a series of nodes where the output of each node is governed by one decision variable of its own, and the output of the previous node.
 
 \subsubsection{Ackley}
 The Ackley test function \citep{jamil2013literature} has been widely used as a benchmark function in the BO literature. It is defined by 
\begin{equation*}
    g(x) = 20\exp\left(-0.2 \sqrt{\frac{1}{D}\sum_{d=1}^D x_d^2}\right) + \exp\left( \frac{1}{D}\sum_{d=1}^D \cos(2\pi x_d)\right) - 20 - e.
\end{equation*}
We adapt it to our setting by letting
\begin{equation*}
    f_1(x) = \frac{1}{D}\sum_{d=1}^D x_d^2;
\end{equation*}
\begin{equation*}
  f_2(x) = \frac{1}{D}\sum_{d=1}^D \cos(2\pi x_d);  
\end{equation*}
\begin{equation*}
 f_3(y_1,y_2) = 20\exp\left(-0.2\sqrt{y_1}\right) + \exp\left(y_2\right) - 20 - e;   
\end{equation*}
$I(1)=I(2)=\{1,\ldots, D\}$; $I(3)=\emptyset$; $J(1)=J(2)=\emptyset$; and $J(3)=\{1,2\}$. In our experiment, we set  $\domain = [-2,2]^D$, and $D=6$.

\subsubsection{Rosenbrock}
 The Rosenbrock test function \citep{jamil2013literature} is also a widely used benchmark function in the BO literature. It is defined by
 \begin{equation*}
     g(x) = -\sum_{d=1}^{D-1}100(x_{d+1}-x_d^2)^2 + (1-x_d)^2.
 \end{equation*}
 We adapt it to our setting by letting
 \begin{equation*}
     f_1(x_1, x_2) = -100(x_{2}-x_1^2)^2 - (1-x_1)^2;
 \end{equation*}
  \begin{equation*}
     f_k(x_k, x_{k+1}, y_{k-1}) = -100(x_{k+1}-x_{k}^2)^2 - (1-x_k)^2 + y_{k-1}, \ k=2,\ldots, D-1;
 \end{equation*}
 $I(k) = \{k, k+1\}, \ k=1,\ldots, D-1$; $J(1)=\emptyset$; and $J(k)=\{k-1\}, \ k=2,\ldots,D-1$. In our experiments, we set $\domain-[-2,2]^D$, and consider $D=3, 5, \textnormal{ and } 7$.
 
 \subsubsection{Drop-Wave}
 The Drop-Wave test function \citep{dropwave} is highly multi-modal and complex. It is defined by
 \begin{equation*}
     g(x) = \frac{1 + \cos\left(12\sqrt{x_1^2 + x_2^2}\right)}{2 + 0.5\left(x_1^2 + x_2^2\right)}.
 \end{equation*}
 We adapt it to our setting by taking
 \begin{equation*}
     f_1(x)= \sqrt{x_1^2 + x_2^2};
 \end{equation*}
 \begin{equation*}
     f_2(y_1) = \frac{1 + \cos\left(12 y_1\right)}{2 + 0.5 y_1^2};
 \end{equation*}
 $I(1)=\{1,2\}$; $I(2)=\emptyset$; $J(1)=\emptyset$; and  $J(2)=\{1\}$. In our experiment, we set $\domain=[-5.12,5.12]^2$.

\subsection{Manufacturing Throughput Maximization}
Here we describe the manufacturing throughput manufacturing test problem. This problem is similar in spirit to the biomanufacturing example in the introduction, but focusing on more traditional manufacturing in which workproduct is discrete rather than continuous. We have a manufacturing line with a series of stations that perform operations: e.g., steel is cut to size, then bent to shape, then holes are drilled, and finally the piece is painted. We consider ``make-to-order' in which custom features of the part (e.g., color, size, orientation of the holes) require waiting until a customer order arrives to begin processing the part.

Orders for parts arrive randomly according to a homogeneous Poisson process. Orders move to the first station in the manufacturing line and enter a queue where they wait. The first order in the queue requires a processing time exponentially distributed with a service rate that decreases with the amount of resource devoted to that station (e.g., more workers, better machines). Parts not being processed wait in the queue until they arrive to the front of the queue. Once processed, a part moves to the second station where it similarly waits in a queue until it is at the front, then waits an exponential amount of time that depends on a second service rate, which we can again control through staffing. This continues, with each part completing service moving to the next queue until it has completed service at all stations.

Our goal is to choose a collection of service rates, one for each station, to maximize the number of parts finished in a fixed amount of time. We constrain the sum of the service rates across the stations to represent a limit on total resources that can be allocated.

In our experiment, we consider a manufacturing line with 4 stations. The objective to maximize is the throughput of the network in steady state, $f(x)$, where $x_i$ is the service rate of station $i$, over the feasible domain $\domain = \{x: 0 \leq x_i , \ i=1,...,4, \textnormal{ and } \sum_{i=1}^4 x_i \leq 1\}$. Let $f_i(x_i, y_{i-1})$ be throughput of station $i$ in steady state given that the service rate of station is $i$ is $x_i$ and station $i-1$ has throughput $y_{i-1}$ in steady state. Then, our objective function can be written as a function network by taking $I(k)=\{k\}, \ k=1,\ldots, K$; $J(1)=\emptyset$; and $J(k)=\{k-1\}, \ k=2,3,4$. This network is illustrated in Figure \ref{fig:manufac}.

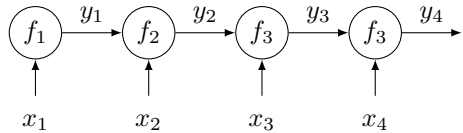
\begin{figure}[h]
  \centering
\resizebox{!}{2cm}{%
\begin{tikzpicture}[
init/.style={
  draw,
  circle,
  inner sep=0.7pt,
  minimum size=0.7cm
},
init2/.style={
  circle,
  inner sep=0.7pt,
  minimum size=0.7cm
},
]
\begin{scope}[start chain=1,node distance=8mm]
\node[on chain=1, init] 
  (f1) {$\f_1$};
\node[on chain=1,init]
 (f2) {$\f_2$};
\node[on chain=1,init] (f3) 
  {$ \f_3$};
  \node[on chain=1,init] (f4) 
  {$ \f_3$};
\node[on chain=1,init2] (f5){};
\end{scope}

\begin{scope}[start chain=2,node distance=8mm]
\node[on chain=2, init2] at (0,-12mm)
(x1) {$x_1$};
\node[on chain=2, init2] 
  (x2) {$x_2$};
\node[on chain=2, init2] 
  (x3) {$x_3$};
\node[on chain=2, init2] 
  (x4) {$x_4$};
\end{scope}

\draw[-latex] (f1) -- (f2)node[pos=0.5,sloped,above] {$y_1$};
\draw[-latex] (f2) -- (f3)node[pos=0.5,sloped,above] {$y_2$};
\draw[-latex] (f3) -- (f4)node[pos=0.5,sloped,above] {$y_3$};
\draw[-latex] (f4) -- (f5)node[pos=0.5,sloped,above] {$y_4$};

\draw[-latex] (x1) -- (f1);
\draw[-latex] (x2) -- (f2);
\draw[-latex] (x3) -- (f3);
\draw[-latex] (x4) -- (f4);
\end{tikzpicture}
}
\caption{Manufacturing line as a function network.\label{fig:manufac}}
\end{figure}

\subsection{Optimization of Pooled Testing for COVID-19}
\label{sec:covid}

Here we describe our COVID-19 testing benchmark problem, which considers reinforcement learning for large-scale testing to prevent the spread of COVID-19.
It builds on the model for infection dynamics described in the epidemic model calibration problem in \S5.3 of the main paper.

\paragraph{Overview}
We design an asymptomatic screening protocol for controlling the spread of COVID-19 in a city.
This approach regularly tests the entire population to identify people who are infected but do not have symptoms and may be unknowingly spreading virus. 
This approach has been employed successfully to control the spread of COVID-19 among students at several US universities \citep{denny2020implementation} and also in Wuhan, China \citep{cao2020post}.
To limit the resources needed for testing, we use {\it pooled testing}, which is described in detail in the supplement. This has a parameter (the {\it pool size}) that, when increased, reduces the testing resources used per person tested but also degrades test accuracy in a complex way that depends on the prevalence of the virus in the population.

We simulate the effect of asymptomatic screening using pooled testing on a single population, indexing time by $t=1,2,3$. As in \S5.3, we track the fraction of the population that is infectious and susceptible. To model the immunity that follows infection with COVID-19, an individual can be ``recovered'' (R), which means that they were previously infected, are no longer infectious and cannot be infected again.  At the start of time period $t$, $I_t$ is the fraction of the population that is infectious, and $R_t$ is the fraction that is recovered. The rest is susceptible.

During each period $t$, the entire population is tested using a pool size of $x_t$. A black-box simulator determines the accuracy of these tests and the testing resources used (which depends on both $x_t$ and the prevalence $I_t$). Individuals testing positive are isolated\footnote{This is sometimes confused with quarantine: close contacts are quarantined, while positives are isolated} so that they cannot infect others during the period, and infectious individuals missed in testing infect others. Lower accuracy results in more individuals missed in testing. At the end of the period, all individuals in isolation are modeled as having recovered and leave isolation. This process results in a loss $L_t$ incorporating infections, testing resources used, and individuals isolated.
Our goal is to choose the pool sizes $x_1$, $x_2$, $x_3$ to minimize the total loss $\sum_t L_t$.

This is encoded as a function network in Figure~3 in the main paper.
As described above and in detail below, each time period performs a calculation that takes the pool size $x_t$ and a bivariate description $(I_t, R_t)$ of the population's current infection status. (The details of this computation is unknown to our function networks Bayesian optimization model.)
It then produces as output a loss $L_t$ and the corresponding description $(I_{t+1}, R_{t+1})$ of the population's infection status at the start of the next period.
The objective function is $\sum_t L_t$.
The known form of the final node $\sum_t L_t$ is leveraged while the other nodes are treated as black boxes.

Given this overview of the problem, we now describe in detail how these black boxes are computed. 

\paragraph{Pooled Testing}
We first describe pooled testing in more detail. Pooled testing is a method for testing a large number of people for the presence of virus or some other pathogen \citep{cleary2021using,dorfman1943detection} in a way that reduces the amount of resource (specifically, chemical reagent and machine time) required per test performed compared with testing each person individually. 

As in any COVID-19 test, we first collect a nasal or saliva sample from each individual being tested. Each sample is placed in a separate tube of fluid. Pooled testing relies on the ability to be able to test several people (a ``pool'') simultaneously, returning a signal that tells us whether (1) no one in the pool is positive; or (2) at least one person in the pool is positive. To accomplish this, a bit of fluid from each sample in the pool is taken and mixed together. Then a single chemical reaction (PCR, or {\it polymerase chain reaction}) is run to asses whether anyone in the pool is positive.

We specifically consider square array pooled testing \citep{westreich2008optimizing}. This approach considers $x^2$ saliva samples as occupying a $x\times x$ grid. Then, it forms $2x$ pools: one pool from the samples in each row; and one from the samples in each column. Pools are tested, as described above. If a sample's row or column pool tests negative, then that sample is considered free of virus. All other samples (those whose row and column pool both test positive) are tested individually using a chemical reaction performed on additional fluid from that sample.
This is illustrated in Figure~3 in the main paper.
%Add a diagram like the illustration in this article
% https://news.cornell.edu/stories/2020/06/group-testing-could-screen-entire-us-research-suggests

The chemical reactions used to check for virus sometimes make errors: both false negatives, in which a pool or individual sample including material from an infected person tests negative; and false positives, in which a pool that does not contain virus nevertheless is deemed positive.  Moreover, the probability of a false negative rises with the pool size \citep{cleary2021using}. This results in errors from the overall pooled testing procedure, where an individual who is virus-free is deemed positive (a false positive) or an individual who is infected with virus is deemed negative (a false negative).

In addition to depending on the pool size, the probability of these two kinds of errors (false positives and false negatives) in the overall testing procedure depends on the prevalence, i.e., the fraction of the population infected. When prevalence is high, there are sometimes two positive individuals providing fluid to a pool. This increases the chance that the pool tests positive. Also, more poools contain positive individuals, increasing the number of negative people whose row and column pools both test positive. This increases the chance that an overall test of a virus-free person will come back positive.

The level of resource used is proportional to the number of chemical reactions performed. This also depends on the pool size and the prevalence  If the prevalence is small, then the number of chemical reactions used is approximately $1/(2x)$ since the number of chemical reactions performed on individual samples is small. As prevalence rises, larger pool sizes require more followup testing (because the pools become likely to contain at least one positive individual) and smaller pool sizes become efficient.

\paragraph{Infection Dynamics without Pooled Testing}
As described above, time is divided into discrete time points $t=1,2,3$, each representing a distinct two-week period. At the start of each period, the population is described by two numbers: $I_t$, the fraction of the population that is infectious; and $R_t$, the fraction of the population that is recovered and cannot be infected again. These numbers are both in $[0,1]$.  The additional $S_t = 1-I_t-R_t$ fraction of the population is susceptible, and can be infected. Such divisions of a population into these three different groups (susceptible, infectious, and recovered) is widely used in epidemiology \citep{covid-white}.

We first describe our assumed infection dynamics in the absence of asymptomatic screening.
This is obtained by integrating continuous-time dynamics within a given two-week period.
We denote time strictly within a two-week period by $t+u$, where $t$ is an integer and $u\in(0,1)$.
During this period, we assume that infectious individuals who were infectious at the start of the period remain so for the full two weeks. Each comes into physical contact with other people at a rate of $\beta$ people per unit time. A fraction $S_t$ are susceptible and become infected\footnote{
Note that we use $S_t$ rather than $S_{t+u}$. This allows for analytical solution to the above equation and does not substantially harm accuracy: in the regimes of importance for solving the benchmark problem optimally, $S_t$ begins close to $1$ and $I_t$ begins close to $0$.  
}. 
This gives us the differential equation
\begin{equation*}
\frac{d}{ds} I_{t+u} = \beta S_t I_{t+u},
\end{equation*}
which has solution:
\begin{equation*}
I_{t+u} = I_t \exp(\beta S_t u).
\end{equation*}

At the end of the two week period, we then assume that all individuals who were infectious at the start of the period convalesce and become recovered. 
Putting this together, in the absence of asymptomatic screening, the resulting dynamics would be:
\begin{align}
S_{t+1} &= S_t - I_t \exp(\beta S_t) \\
I_{t+1} &= I_t (\exp(\beta S_t) - 1)\\
R_{t+1} &= R_t + I_t
\end{align}

Although the details of these dynamics are different from those in \S5.3, it results in behavior that is qualitatively similar. In particular, if $\beta$ is small enough, then $I_t$ shrinks to 0, but if it is large enough then the fraction of the population infected grows to a high fraction over a small number of time periods.

% In detail, each infectious person comes into physical contact with $\beta$ people. A fraction $S_t$ are susceptible and become infected. This results in infecting a fraction $\beta S_t I_t$ of the population, increasing $I_t$ by this amount and decrementing $S_t$ by the same amount.

% Simultaneously, a fraction $\gamma$ of the infectious individuals recover, which results in the fraction $\gamma I_t$ of the population moving from the infectious state to the recovered state.

% Putting this together, in the absence of asymptomatic screening, the resulting dynamics would be:
% \begin{align}
% S_{t+1} &= S_t - \beta S_t I_t \\
% I_{t+1} &= I_t + \beta S_t I_t - \gamma I_t\\
% R_{t+1} &= R_t + \gamma I_t
% \end{align}

In our implementation, we set $\beta = (14/3) \ln(2) \approx 3.23$, corresponding to an epidemic that doubles in size every 3 days in the absence of any interventions.

We now incorporate the effect of asymptomatic screening.

\paragraph{Infection Dynamics with Pooled Testing}
Our simulation includes pooled testing as follows. 
Pooled testing using the pool size $x_t$ is used at the start of the period. As described above, its error rates (false positive and false negative) and its efficiency depend on both the pool size and the prevalence ($I_t$). We use a black-box computation using logic described above to calculate three quantities: 
\newcommand{\FPR}{\alpha_{\mathrm{FP}}} % false positive rate
\newcommand{\TPR}{\alpha_{\mathrm{TP}}} % true positive rate
\newcommand{\tests}{\alpha_{\mathrm{C}}} % number of tests

\begin{itemize}
\item $\FPR(x_t, I_t)$, the fraction of virus-free individuals tested that test positive (i.e., the false positive rate for the overall pooled testing procedure);
\item $\TPR(x_t, I_t)$, the fraction of infected individuals tested that test positive (i.e., the true positive rate for the overall pooled testing procedure);
\item $\tests(x_t, I_t)$, the number of chemical reactions performed across the entire population.
\end{itemize}

Individuals that test positive are immediately removed from the population placed into isolation.  This includes both infectious individuals 
(in particular, a fraction $(\TPR(x_t,I_t))I_t$ of the overall population) as well as susceptible and recovered individuals who were incorrectly classified (fractions $\FPR(x_t,I_t) S_t$ and $\FPR(x_t,I_t)R_t$ of the overall population respectively).
Thus, the number of people isolated in period $t$ is, 
\begin{equation*}
Q_t = \TPR(x_t,I_t)I_t + \FPR(x_t,I_t) (S_t + R_t).
\end{equation*}
This results in a term $c_Q Q_t$ that is added to our loss, representing the social costs of isolation.

Because some infectious individuals are in isolation, the number of new infections is smaller than in the setting described above without asymptomatic screening. This number is 
$I_t (1-\TPR(x_t, I_t)$. Following the infection dynamics described above, this results in
an additional new 
$I_t (1-\TPR(x_t, I_t)) \exp(\beta S_t)$ infections drawn from the susceptible population. In addition, all individuals who were infectious as the start of period $t$ recover. 
Thus, our dynamics are:
\begin{align*}
I_{t+1} &= I_t (1-\TPR(x_t, I_t)) \exp(\beta S_t) \\
S_{t+1} &= S_t - I_{t+1}\\
R_{t+1} &= R_t + I_t
\end{align*}

One may wonder about two modeling details.
First, susceptible people who are erroneously in isolation are nevertheless modeled as eligible for infection. Additionally, recovered people are modeled as being tested, although in practice one might choose to not test these individuals. 
These assumptions have little impact on outcomes because (1) false positive rates are small enough that the fraction of the susceptible population in isolation is a very small fraction of the overal susceptible population; (2) in the regimes where good solutions lie, few people are ever infected, making the recovered population also small.
Making these assumptions simplifies the description and implementation.

The loss at time $t$ is the sum of the social cost of isolation described above, the cost of the testing supplies consumed $c_T \tests(x_t,I_t)$, and the social cost associated with the new infections, $c_I I_{t+1}$.
\begin{equation*}
 L_{t} = c_T \tests(x_t,I_t) +  c_Q Q_t +  c_I I_{t+1}.
\end{equation*}

\section{Additional Numerical Experiment: Active Learning for Robot Pushing}
\label{supp:robot_push}

Here we describe one additional experiment. We consider a variation of the active learning for robot pushing problem introduced by \cite{wang2017max} whose goal is to teach a robot to push an object to a predetermined target location. We modify the problem by allowing the robot to push the object several times instead of only once. We formalize this problem as follows. Let $x_\mathrm{init}, x_\mathrm{target}\in[-5,5]^2$ denote the object's initial and target locations, respectively. At each time step, $t$, we choose the location of the robot's arm, $r_t\in[-5,5]^2$, and the duration of the push, $d_t\in[1,12]$. The robot then moves its arm from $r_t$ in the current direction of the object, $x_t$, over $d_t$ units of time. After this push, the location of the object becomes $x_{t+1}$ (if the robot fails to push the object, $x_{t+1} = x_t$). The goal is to choose $(r_t, d_t)$ for $t=1,\ldots, T$ to minimize $\|x_\mathrm{target} - x_T\|_2^2$. We set $x_\mathrm{init}=(0,0)$, $x_\mathrm{target}=(2.9,1.6)$, and $T=3$. This can be interpreted as a function network by associating each time step with a pair of node functions $p_{t,1}$ and $p_{t,1}$ which take $(r_t,d_t,x_t)$ as input and produce $x_{t+1} = (p_{t,1}(r_t,d_t,x_t), p_{t,2}(r_t,d_t,x_t))$ as output.

The results of this experiment are shown in Figure~\ref{fig:robotpush}. EI-FN improves over EI-CF and improves substantially over EI and Random.

% In our description, we use exp(beta)
% In the code, we have exp(beta) = alpha^test_period
% test_period = 14 days
% alpha = 2^(1/doubling time)
% doubling time = 3 days
% So, alpha^test_period = (2^(1/3))^14 = 2^(14/3)
% beta = ln(alpha^test_period) = ln(2^(14/3))
% = (14/3)*ln(2)
% = 3.23468684

  %\begin{figure}[h]
  %\centering
  %\includegraphics[width=0.32\textwidth]{figures/fetch1.jpg}
  %\includegraphics[width=0.32\textwidth]{figures/fetch2.jpg}
  %\includegraphics[width=0.32\textwidth]{figures/fetch3.jpg}
%\end{figure}

\begin{figure}[tb]
\centering
\includegraphics[height=0.3\textwidth]{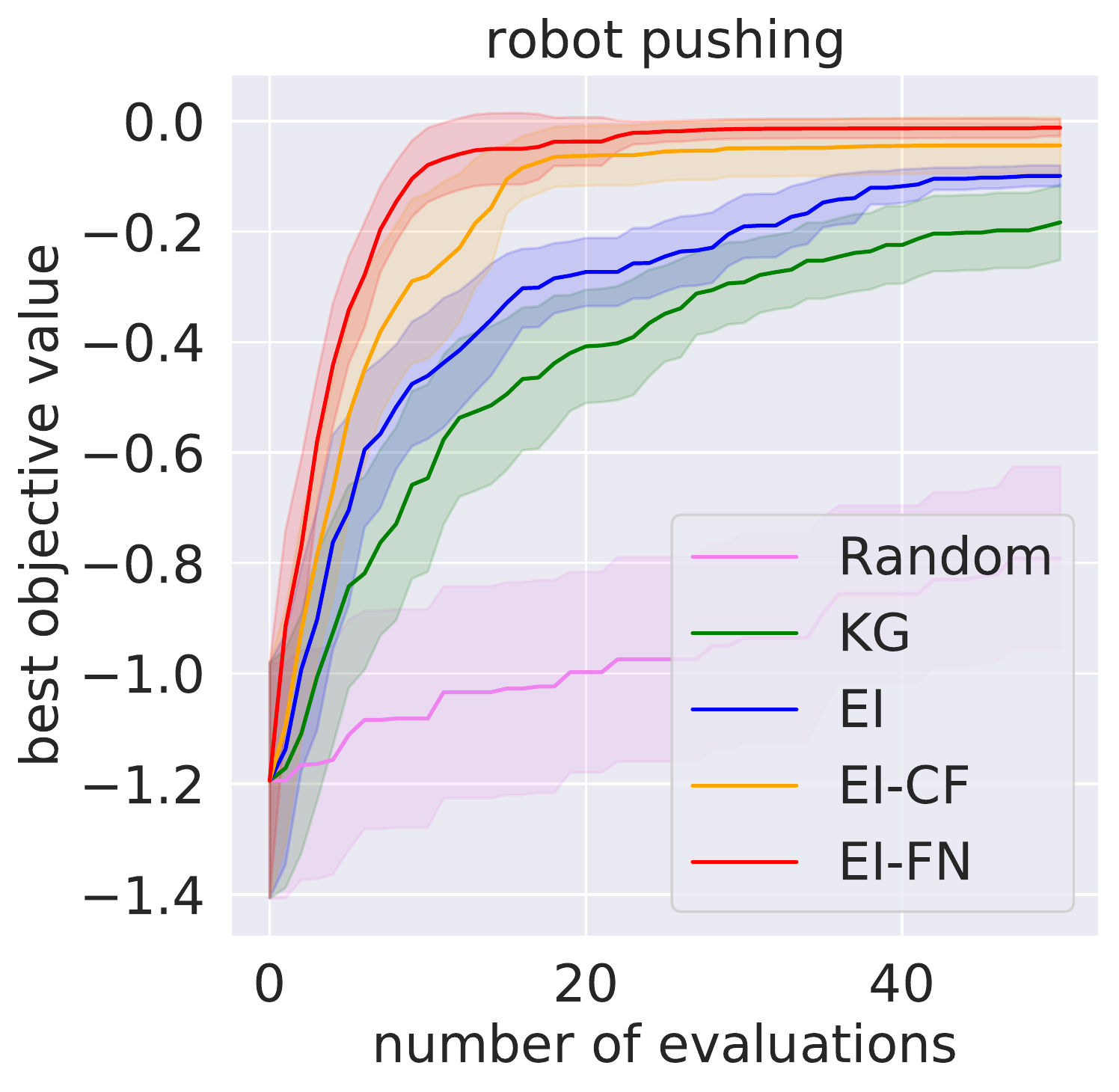}

\caption{
Results from the experiment described in \S\ref{supp:robot_push}. 
\label{fig:robotpush}
}
\end{figure}

 \section{Posterior Mean and Covariance Functions}
 \label{supp:posterior}
In this section, we write explicit formulas for the posterior mean and covariance functions of the GP distributions associated to the node functions, $f_1, \ldots, f_K$. To write this more simply, we define some additional notation.
Given generic vectors $x\in\R^D$ and $y\in \R^K$, we define $z_{k} = (x_{I(k)}, y_{J(k)})$ as  the elements of these vectors supplied as 
input to node $k$.
Similarly, given a historical observation of the values of the node functions $y_\ell = (h_1(x_\ell), \ldots, h_K(x_\ell)), \ \ell=1,\ldots, n$, we define $z_{\ell,k} = (x_{\ell,I(k)}, y_{\ell, J(k)})$. Using this notation, our posterior mean and covariance functions can be written as
\begin{align*}
  \mu_{n,k}(z_k)  =  \mu_{0,k}(z_k) +
  \Sigma_{0,k}\left(z_k,z_{1:n,k}\right)\Sigma_{0,k}\left(z_{1:n,k}, z_{1:n,k}\right)^{-1}\left(y_{1:n,k} - \mu_{0,k}\left(z_{1:n,k}\right)\right),
\end{align*}
and
\begin{align*}
 \Sigma_{n,k}(z_k, z_k')  =  \Sigma_{0,k}(z_k, z_k') -
 &\Sigma_{0,k}\left(z_k,z_{1:n, k}\right)\Sigma_{0,k}\left(z_{1:n, k}, z_{1:n, k}\right)^{-1}\Sigma_{0,k}\left(z_{1:n, k},z_k'\right),
\end{align*}
respectively.

\end{document}